\documentclass{article}

\PassOptionsToPackage{numbers}{natbib}

\usepackage[final]{neurips_2022}

\usepackage[utf8]{inputenc} 
\usepackage[T1]{fontenc}    
\usepackage{hyperref}       
\usepackage{url}            
\usepackage{booktabs}       
\usepackage{amsfonts}       
\usepackage{nicefrac}       
\usepackage{microtype}      
\usepackage{xcolor}         

\usepackage[ruled, linesnumbered]{algorithm2e}
\usepackage{algpseudocode}
\usepackage{amsmath}
\usepackage{amsthm}
\usepackage{bbm}
\usepackage{caption}
\usepackage[capitalize]{cleveref}
\usepackage{enumitem}
\usepackage{graphicx}
\usepackage{nicefrac}
\usepackage{subcaption}
\usepackage{wrapfig}

\title{Non-monotonic Resource Utilization in the Bandits with Knapsacks Problem}

\author{%
  Raunak Kumar \\
  Department of Computer Science \\
  Cornell University \\
  Ithaca, NY 14853 \\
  \texttt{raunak@cs.cornell.edu} \\
  \And
  Robert D. Kleinberg \\
  Department of Computer Science \\
  Cornell University \\
  Ithaca, NY 14853 \\
  \texttt{rdk@cs.cornell.edu} \\
}

\newtheorem{theorem}{Theorem}[section]
\newtheorem{lemma}{Lemma}[section]
\newtheorem{corollary}{Corollary}[section]

\theoremstyle{definition}
\newtheorem{definition}{Definition}[section]

\theoremstyle{plain}

\newcommand{\E}{\mathbb{E}}
\renewcommand{\Pr}{{\bf Pr}}

\newcommand{\bwk}{\mathsf{BwK}}
\newcommand{\lp}{\mathsf{LP}}
\newcommand{\opt}{\mathsf{OPT}}
\newcommand{\rew}{\mathsf{REW}}

\newcommand{\lcb}{\mathsf{LCB}}
\newcommand{\ucb}{\mathsf{UCB}}
\newcommand{\rad}{\mathsf{rad}}

\newcommand{\poly}{\mathrm{poly}}

\newcommand{\calA}{{\cal A}}
\newcommand{\calJ}{{\cal J}}
\newcommand{\calX}{{\cal X}}

\newcommand{\deltadrift}{\delta_{\text{drift}}}
\newcommand{\deltasupport}{\delta_{\text{support}}}
\newcommand{\deltaslack}{\delta_{\text{slack}}}

\newcommand{\indicator}{\mathbbm{1}}

\begin{document}

\maketitle

\begin{abstract}
Bandits with knapsacks ($\bwk$)~\citep{badanidiyuru2018bandits} is an
influential model of sequential decision-making under uncertainty that
incorporates resource consumption constraints. In each round, the
decision-maker observes an outcome consisting of a reward and a vector of
nonnegative resource consumptions, and the budget of each resource is
decremented by its consumption. In this paper we introduce a natural
generalization of the stochastic $\bwk$ problem that allows non-monotonic
resource utilization. In each round, the decision-maker observes an outcome
consisting of a reward and a vector of resource \emph{drifts} that can be
positive, negative or zero, and the budget of each resource is incremented by
its drift. Our main result is a Markov decision process (MDP) policy that has
\emph{constant} regret against a linear programming (LP) relaxation when the
decision-maker \emph{knows} the true outcome distributions. We build upon this
to develop a learning algorithm that has \emph{logarithmic} regret against the
same LP relaxation when the decision-maker \emph{does not know} the true
outcome distributions. We also present a reduction from $\bwk$ to our model
that shows our regret bound matches existing results~\citep{li2021symmetry}.
\end{abstract}

\section{Introduction}
\label{sec:introduction}

Multi-armed bandits are the quintessential model of sequential decision-making
under uncertainty in which the decision-maker must trade-off between exploration
and exploitation. They have been studied extensively and have numerous
applications, such as clinical trials, ad placements, and dynamic pricing to
name a few. We refer the reader
to~\citet{bubeck2012regret,slivkins2019introduction,lattimore2020bandit} for an
introduction to bandits. An important shortcoming of the basic stochastic
bandits model is that it does not take into account resource consumption
constraints that are present in many of the motivating applications. For
example, in a dynamic pricing application the seller may be constrained by a
limited inventory of items that can run out well before the end of the time
horizon. The bandits with knapsacks $(\bwk$)
model~\citep{tran2010epsilon,tran2012knapsack,badanidiyuru2013bandits,badanidiyuru2018bandits}
remedies this by endowing the decision-maker with some initial budget for each
of $m$ resources. In each round, the outcome is a reward and a vector of
nonnegative resource consumptions, and the budget of each resource is
decremented by its consumption. The process ends when the budget of any resource
becomes nonpositive. However, even this formulation fails to model that in many
applications resources can get replenished or renewed over time. For example, in
a dynamic pricing application a seller may receive shipments that increase their
inventory level.

\paragraph{Contributions}
In this paper we introduce a natural generalization of $\bwk$ by allowing
non-monotonic resource utilization. The decision-maker starts with some initial
budget for each of $m$ resources. In each round, the outcome is a reward and a
vector of resource \emph{drifts} that can be positive, negative or zero, and the
budget of each resource is incremented by its drift. A negative drift has the
effect of decreasing the budget akin to consumption in $\bwk$ and a positive
drift has the effect of increasing the budget. We consider two settings: (i)
when the decision-maker \emph{knows} the true outcome distributions and must
design a Markov decision process (MDP) policy; and (ii) when the decision-maker
\emph{does not know} the true outcome distributions and must design a learning
algorithm.

\begin{wrapfigure}{r}{0.45\textwidth}
  \centering
  \includegraphics[scale=0.30]{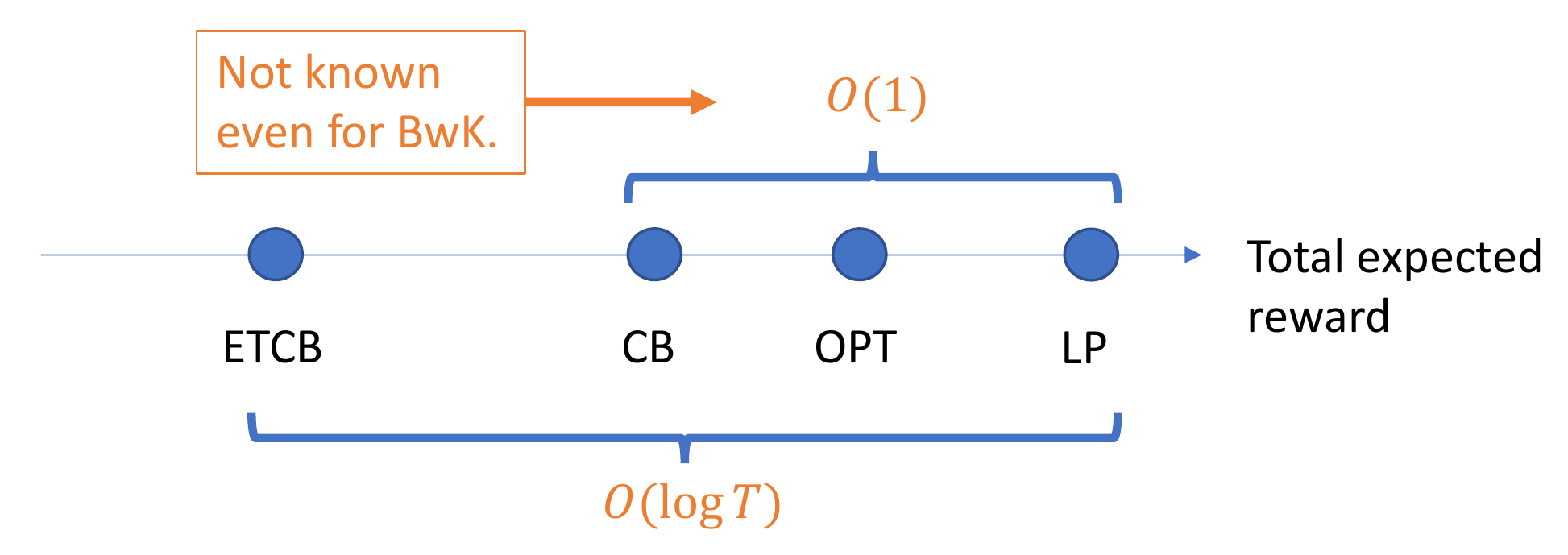}
\end{wrapfigure}
Our main contribution is an MDP policy, \texttt{ControlBudget(CB)}, that has
\emph{constant} regret against a linear programming (LP) relaxation. Such a
result was not known even for $\bwk$. We build upon this to develop a learning
algorithm, \texttt{ExploreThenControlBudget(ETCB)}, that has \emph{logarithmic}
regret against the same LP relaxation. We also present a reduction from $\bwk$
to our model and show that our regret bound matches existing results.


Instead of merely sampling from the optimal probability distribution over arms,
our policy samples from a perturbed distribution to ensure that the budget of
each resource stays close to a decreasing sequence of thresholds. The sequence
is chosen such that the expected leftover budget is a constant and proving this
is a key step in the regret analysis. Our work combines aspects of related work
on logarithmic regret for $\bwk$~\citep{flajolet2015logarithmic,li2021symmetry}.

\paragraph{Related Work}
Multi-armed bandits have a rich history and logarithmic instance-dependent
regret bounds have been known for a long
time~\citep{lai1985asymptotically,auer2002finite}. Since then, there have been
numerous papers extending the stochastic bandits model in a variety of
ways~\citep{auer2002nonstochastic,slivkins2011contextual,kleinberg2019bandits,badanidiyuru2018bandits,immorlica2019adversarial,agrawal2014bandits,agrawal2016linear}.

To the best of our knowledge, there are three papers on logarithmic regret
bounds for $\bwk$. \citet{flajolet2015logarithmic} showed the first logarithmic
regret bound for $\bwk$. In each round, their algorithm finds the optimal basis
for an optimistic version of the LP relaxation, and chooses arms from the
resulting basis to ensure that the average resource consumption stays close to a
pre-specified level. Even though their regret bound is logarithmic in $T$ and
inverse linear in the suboptimality gap, it is exponential in the number of
resources. \citet{li2021symmetry} showed an improved logarithmic regret bound
that is polynomial in the number of resources, but it scales inverse
quadratically with the suboptimality gap and their definition of the gap is
different from the one in~\citet{flajolet2015logarithmic}. The main idea behind
improving the dependence on the number of resources is to proceed in two phases:
(i) identify the set of arms and binding resources in the optimal solution; (ii)
in each round, solve an adaptive, optimistic version of the LP relaxation and
sample an arm from the resulting probability distribution.
Finally,~\citet{sankararaman2021bandits} show a logarithmic regret bound for
$\bwk$ against a \emph{fixed-distribution} benchmark. However, the regret of
this benchmark itself with the optimal MDP policy can be as large as
$O(\sqrt{T})$~\citep{flajolet2015logarithmic,li2021symmetry}.

\section{Preliminaries}
\label{sec:preliminaries}



\subsection{Model}
\label{subsec:model}

Let $T$ denote a finite time horizon, $\calX = \{ 1, \dots, k \}$ a set of $k$
arms, $\calJ = \{ 1, \dots, m \}$ denote a set of $m$ resources, and $B_{0,j} =
B$ denote the initial budget of resource $j$. In each round $t \in [T]$, if the
budget of any resource is less than $1$, then $\calX_t = \{ 1 \}$. Otherwise,
$\calX_t = \calX$. The algorithm chooses an arm $x_t \in \calX_t$ and observes
an outcome $o_t = (r_t, d_{t,1}, \dots, d_{t,m}) \in [0,1] \times [-1,1]^m$. The
algorithm earns reward $r_t$ and the budget of resource $j \in \calJ$ is
incremented by drift $d_{t,j}$ as $B_{t,j} = B_{t-1,j} + d_{t,j}$.

Each arm $x \in \calX$ has an outcome distribution over $[0,1] \times [-1,1]^m$
and $o_t$ is drawn from the outcome distribution of the arm $x_t$. We use
$\mu_x^o = (\mu_x^r, \mu_x^{d,1}, \dots, \mu_x^{d,m})$ to denote the expected
outcome vector of arm $x$ consisting of the expected reward and the expected
drifts for each of the $m$ resources.\footnote{In fact, our proofs remain valid
even if the outcome distribution depends on the past history provided the
conditional expectation is independent of the past history and fixed for each
arm. In this case $\mu_x^o$ denotes the conditional expectation of $o_t$ when
arm $x$ is pulled in round $t$. Since our proofs rely on the Azuma-Hoeffding
inequality, we need this assumption on the conditional expectation to hold.}  We
also use $\mu^{d,j} = (\mu_x^{d,j} : x \in \calX)$ to denote the vector of
expected drifts for resource $j$. We assume that arm $x^0 = 1 \in \calX$ is a
null arm with three important properties: (i) its reward is zero a.s.; (ii) the
drift for each resource is nonnegative a.s.; and (iii) the expected drift for
each resource is positive. The second and third properties of the null arm plus
the model's requirement that $x_t = 1$ if $\exists j$ s.t.\ $B_{t-1, j} < 1$
ensure that the budgets are nonnegative a.s.\ and can be safely increased from
$0$.

Our model is intended to capture applications featuring resource renewal, such
as the following.  In each round, each resource gets replenished by some random
amount and the chosen arm consumes some random amount of each resource. If the
consumption is less than replenishment, the resource gets renewed. The random
variable $d_{t,j}$ then models the net replenishment minus consumption. The full
model presented above is more general because it allows both the consumption and
replenishment to depend on the arm pulled.

We consider two settings in this paper.
\begin{description}
  \item[MDP setting] The decision-maker \emph{knows} the true outcome
  distributions. In this setting the model implicitly defines an MDP, where the
  state is the budget vector, the actions are arms, and the transition
  probabilities are defined by the outcome distributions of the arms.
  \item[Learning setting] The decision-maker \emph{does not know} the
  true outcome distributions.
\end{description}

The goal is to design to an MDP policy for the first setting and a learning
algorithm for the second, and bound their regret against an LP relaxation as
defined in the next subsection.


\subsection{Linear Programming Relaxation}

Similar to~\citet[Lemma 3.1]{badanidiyuru2018bandits}, we consider the following
LP relaxation that provides an upper bound on the total expected reward of any
algorithm:
\begin{equation}\label{eq:lp_relaxation}
  \opt_\lp = \max_{p} \left \{ \sum_{x \in \calX} p_x \mu_x^r : \sum_{x \in \calX} p_x \mu_x^{d,j} \geq \nicefrac{-B}{T} \ \forall j \in \calJ, \ \sum_{x \in \calX} p_x = 1, \ p_x \geq 0 \ \forall x \in \calX  \right \}.
\end{equation}

\begin{lemma}
  The total expected reward of any algorithm is at most $T \cdot \opt_\lp$.
\end{lemma}
The proof of this lemma, similar to those in existing
works~\citep{agrawal2014bandits,badanidiyuru2018bandits}, follows from the
observations that (i) the variables $p = \{ p_x : x \in \calX \}$ can be
interpreted as the probability of choosing arm $x$ in a round; and (ii) if we
set $p_x$ equal to the expected number of times $x$ is chosen by an algorithm
divided by $T$, then it is a feasible solution for the LP.

\begin{definition}[Regret]
  The regret of an algorithm $\calA$ is defined as $R_T(\calA) = T \cdot
  \opt_\lp - \rew(\calA)$, where $\rew(\calA)$ denotes the total expected reward
  of $\calA$.
\end{definition}


%
%


\subsection{Assumptions}
\label{subsec:assumptions}

We assume that the initial budget of every resource is $B \leq T$. This
assumption is without loss of generality because otherwise we can scale the
drifts by dividing them by the smallest budget. This results in a smaller
support set for the drift distribution that is still contained in $[-1, 1]$.

Our assumptions about the null arm $x^0$ are a major difference between our
model and $\bwk$. In $\bwk$ the budgets can only decrease and the process ends
when the budget of any resource reaches $0$. However, in our model the budgets
can increase or decrease, and the process ends at the end of the time horizon.
Our assumptions about the null arm allow us to increase the budget from $0$
without making it negative.\footnote{In a model where, in each round, each
resource gets replenished by some random amount and the chosen arm consumes some
random amount of each resource, the null arm represents the option to remain
idle and do nothing while waiting for resource replenishment.
See~\cref{sec:appendix_assumptions} for more discussion on the assumptions about
the null arm.} A side-effect of this is that in our model we can even assume
that $B$ is a small constant because we can always increase the budget by
pulling the null arm, in contrast to existing literature on $\bwk$ that assume
the initial budgets are large and often scale with the time horizon.

A standard assumption for achieving logarithmic regret in stochastic bandits is
that the gap between the total expected reward of an optimal arm and that of a
second-best arm is positive. There are a few different ways in which one could
translate this to our model where the optimal solution is a mixture over arms.
We make the following choice. We assume that there exists a unique set of arms
$X^*$ that form the support set of the LP solution and a unique set of resources
$J^*$ that correspond to binding constraints in the LP
solution~\citep{li2021symmetry}. We define the gap of the problem instance
in~\cref{def:gap} and our uniqueness assumption\footnote{This assumption is
essentially without loss of generality because the set of problem instances with
multiple optimal solutions is a set of measure zero.} implies that the gap is
strictly positive.

We make a few separation assumptions parameterized by four positive constants
that can be arbitrarily small. First, the smallest magnitude of the drifts,
$\deltadrift = \min \{ |\mu_x^{d,j}| : x \in \calX, j \in \calJ \}$, satisfies
$\deltadrift > 0$. Second, the smallest singular value of the LP constraint
matrix, denoted by $\sigma_{\min}$, satisfies $0 < \sigma_{\min} < 1$. Third,
the LP solution $p^*$ satisfies $p^*_x \geq \deltasupport > 0$ for all $x \in
X^*$. Fourth, $\sum_{x \in X^*} p^*_x \mu_x^{d,j} \geq \deltaslack > 0$ for all
resources $j \notin J^*$. The first assumption is necessary for logarithmic
regret bounds because otherwise one can show that the regret of the optimal
algorithm for the case of one resoure and one zero-drift arm is
$\Theta(\sqrt{T})$~(\cref{sec:appendix_zero_drift}). The second and third
assumptions are essentially the same as in existing literature on logarithmic
regret bounds for $\bwk$~\citep{flajolet2015logarithmic,li2021symmetry}. The
fourth assumption allows us to design algorithms that can increase the budgets
of the non-binding resources away from $0$, thereby reducing the number of times
the algorithm has to pull the null arm. Otherwise, if they have zero drift,
then, as stated above, the regret of the optimal algorithm for the case of one
resource and one zero-drift arm is
$\Theta(\sqrt{T})$~(\cref{sec:appendix_zero_drift}).

\section{MDP Policy with Constant Regret}
\label{sec:known_distributions}

In this section we design an MDP policy,
\texttt{ControlBudget}~(\cref{alg:controlbudget}), with constant regret in terms
of $T$ for the setting when the learner knows the true outcome distributions and
our model implicitly defines an MDP~(\cref{subsec:model}).  At a high level,
\texttt{ControlBudget}, which shares similarities with~\citet[Algorithm
UCB-Simplex]{flajolet2015logarithmic}, plays arms to keep the budgets close to a
decreasing sequence of thresholds. The choice of this sequence allows us to show
that the expected leftover budgets and the expected number of null arm pulls are
constants. This is a key step in proving the final regret bound. We start by
considering the special case of one resource
in~\cref{subsec:policy_one_resource} because it provides intuition for the
general case of multiple resources in~\cref{subsec:policy_multiple_resources}.


\subsection{Special Case: One Resource}
\label{subsec:policy_one_resource}

Since there is only one resource we drop the superscript $j$ in this section. We
say that an arm $x$ is a positive (resp.\ negative) drift arm if $\mu_x^d > 0$
(resp.\ $\mu_x^d < 0$). The following lemma characterizes the possible
solutions of the LP~(\cref{eq:lp_relaxation}).

\begin{lemma}\label{lemma:lp_structure_one_resource}
  The solution of the LP relaxation~(\cref{eq:lp_relaxation}) is supported on at
  most two arms. Furthermore, if $T \geq \nicefrac{B}{\deltadrift}$, then the
  solution belongs to one of three categories: (i) supported on a single
  positive drift arm; (ii) supported on the null arm and a negative drift arm;
  (iii) supported on a positive drift arm and a negative drift arm.
\end{lemma}

The proof of this lemma follows from properties of LPs and a case analysis of
which constraints are tight.  Our MDP policy,
\texttt{ControlBudget}~(\cref{alg:controlbudget_one_resource}), deals with the
three cases separately and satisfies the following regret bound.\footnote{In
this theorem and the rest of the paper, we use $\tilde{C}$ to denote a constant
that depends on problem parameters, including $k, m$, and the various separation
constants mentioned in~\cref{subsec:assumptions}, but \emph{does not depend on}
$T$. We use this notation because the main focus of this work is how the regret
scales as a function of $T$.}

\begin{algorithm}[ht]
  \caption{\texttt{ControlBudget} (for $m = 1$)}
  \label{alg:controlbudget_one_resource}
  \DontPrintSemicolon
  \SetKwInOut{Input}{Input}
  \SetKwInOut{Output}{Output}
  \KwIn{time horizon $T$, initial budget $B$, set of arms $\calX$, set of resources $\calJ$, constant $c > 0$.}

  Set $B_0 = B$.

  \uIf{LP solution is supported on positive drift arm $x^p$}{

    \For{$t = 1, 2, \dots, T$}{

      If $B_{t-1} < 1$, pull $x^0$. Otherwise, pull $x^p$.

    }

  }
  \uElseIf{LP solution is supported on null arm $x^0$ and negative drift arm $x^n$}{

    \For{$t = 1, 2, \dots, T$}{

      Define threshold $\tau_t = c \log(T-t)$.

      If $B_{t-1} < \max \{ 1, \tau_t \}$, pull $x^0$. Otherwise, pull $x^n$.

    }

  }
  \uElseIf{LP solution is supported on positive drift arm $x^p$ and negative drift arm $x^n$}{

    \For{$t = 1, 2, \dots, T$}{

      Define threshold $\tau_t = c \log(T-t)$.

      If $B_{t-1} < 1$, pull $x^0$. If $1 \leq B_{t-1} < \tau_t$, pull $x^p$. Otherwise, pull $x^n$.

    }

  }
\end{algorithm}

\begin{theorem}\label{theorem:regret_controlbudget_one_resource}
  If $c \geq \frac{6}{\deltadrift^2}$, the MDP policy
  \texttt{ControlBudget}~(\cref{alg:controlbudget_one_resource}) satisfies
  \begin{equation}
    R_T(\texttt{ControlBudget}) \leq \tilde{C},
  \end{equation}
  where $\tilde{C} =  O \left(\deltadrift^{-4} \ln \left( \left( 1 - \exp \left( -
  \frac{\deltadrift^2}{8} \right) \right)^{-1} \right) + \deltadrift^{-1} \left( 1 -
  \exp \left( \deltadrift^2 \right) \right)^{-2} \right)$ is a constant.
\end{theorem}

We defer all proofs in this section to~\cref{sec:appendix_known_one_resource},
but we include a proof sketch of most results in the main paper following the
statement. The proof of~\cref{theorem:regret_controlbudget_one_resource} follows
from the following sequence of lemmas.

\begin{lemma}\label{lemma:one_arm_pos_drift}
  If the LP solution is supported on a positive drift arm $x^p$, then
  \begin{equation}
    R_T(\texttt{ControlBudget}) \leq \tilde{C},
  \end{equation}
  where $\tilde{C} = O \left( \deltadrift^{-3} \ln \left( \left( 1 - \exp
  \left( - \frac{\deltadrift^2}{8} \right) \right)^{-1} \right) \right)$ is a constant.
\end{lemma}
We can write the regret in terms of the norm of $\xi = (\xi_{x^p})$, where
$\xi_{x^p}$ is the expected difference between the number of times $x^p$ is
played by the LP and by \texttt{ControlBudget}. This is equal to the expected
number of times the policy plays the null arm and, in turn, is equal to the
expected number of rounds in which the budget is below $1$. Since both $x^0$ and
$x^p$ have positive drift, this is a transient random walk that drifts away from
$0$. It is known that such a walk spends a constant number of rounds in any
state in expectation.

\begin{lemma}\label{lemma:regret_one_resource_case_2}
  If the LP solution is supported on the null arm $x^0$ and a negative drift arm
  $x^n$, then 
  \begin{equation}
    R_T(\texttt{ControlBudget}) \leq \tilde{C} \cdot \E [B_T],
  \end{equation}
  where $\tilde{C} = O(\deltadrift^{-1})$ is a constant.
\end{lemma}
We can write the regret in terms of the norm of $\xi =  (\xi_{x^0}, \xi_{x^n})$,
where $\xi_x$ is the expected difference between the number of times $x$ is
played by the LP and by \texttt{ControlBudget}. Since both constraints (resource
and sum-to-one) are tight, the lemma follows by writing $\xi = D^{-1} b$ and
taking norms, where $D$ is the LP constraint matrix and $b = (-\E [B_T], 0)$.

\begin{lemma}\label{lemma:regret_one_resource_case_3}
  If the LP solution is supported on a positive drift arm $x^p$ and a negative
  drift arm $x^n$, then
  \begin{equation}
    R_T(\texttt{ControlBudget}) \leq \tilde{C} \cdot \max \{ \E [B_T], \E [N_{x^0}] \},
  \end{equation}
  where $\E [N_{x^0}]$ denotes the expected number null arm pulls and $\tilde{C}
  = O(\deltadrift^{-1})$ is a constant.
\end{lemma}
This lemma follows similarly to the previous one by writing regret in terms of
the norm of $\xi =  (\xi_{x^p}, \xi_{x^n})$ and writing $\xi = D^{-1} b$ for $b
= (-\E [B_T], \E [N_{x^0}])$.

Therefore, proving that $R_T(\texttt{ControlBudget})$ is a constant in $T$
requires proving that both the expected leftover budget and expected number of
null arm pulls are constants. Intuitively, we could ensure $\E [B_T]$ is small
by playing the negative drift arm whenever the budget is at least $1$.  However,
there is constant probability of the budget decreasing below $1$ and the
expected number of null arm pulls becomes $O(T)$. \texttt{ControlBudget} solves
the tension between the two objectives by carefully choosing a decreasing
sequence of thresholds $\tau_t$. The threshold is initially far from $0$ to
ensure low probability of pre-mature resource depletion, but decreases to $0$
over time to ensure small expected leftover budget and decreases at a rate that
ensures the expected number of null arm pulls is a constant. 

\begin{lemma}\label{lemma:null_arm_pulls_one_resource}
  If the LP solution is supported on a positive drift arm $x^p$ and a negative
  drift arm $x^n$, and $c \geq \frac{6}{\deltadrift^2}$, then
  \begin{equation}
    \E [N_{x^0}] \leq \tilde{C},
  \end{equation}
  where $\tilde{C} = O \left( \deltadrift^{-3} \ln \left( \left( (1 - \exp
  \left( - \frac{\deltadrift^2}{8} \right) \right)^{-1} \right) \right)$ is a constant.
\end{lemma}
If the budget is below the threshold, i.e., $B_{t-1} < \tau_t$ for some $t$,
then \texttt{ControlBudget} pulls $x^p$ until $B_s \geq \tau_{s+1}$ for some $s
\geq t$. Since $x^p$ has positive drift, the event that repeated pulls decrease
the budget towards $0$ is a low probability event. Using this, our choice of
$\tau_t = c \log(T-t)$ for an appropriate constant $c$, and summing over all
rounds shows that the expected number of rounds in which the budget is less than
$1$ is a constant in $T$.

\begin{lemma}\label{lemma:leftover_budget_one_resource}
  If the LP solution is supported on two arms, and $c \geq
  \frac{6}{\deltadrift^2}$, then
  \begin{equation}
    \E [B_T] \leq \tilde{C},
  \end{equation}
  where $\tilde{C} = \tilde{O} \left( \left( 1 - \exp \left( \deltadrift^2
  \right) \right)^{-2} + \deltadrift^{-2} \right)$ is a constant.
\end{lemma}
If $B_{t-1} \geq \tau_t$, then \texttt{ControlBudget} pulls a negative drift arm
$x^n$. We can upper bound the expected leftover budget by conditioning on $q$,
the number of consecutive pulls of $x^n$ at the end of the timeline. The main
idea in completing the proof is that (i) if $q$ is large, then it corresponds to
a low probability event; and (ii) if $q$ is small, then the budget in round
$T-q$ was smaller than $\tau_q$, which is a decreasing sequence in $q$, and
there are few rounds left so the budget cannot increase by too much.


\subsection{General Case: Multiple Resources}
\label{subsec:policy_multiple_resources}

Now we use the ideas from~\cref{subsec:policy_one_resource} to tackle the case
of $m > 1$ resources that is much more challenging.
Generalizing~\cref{lemma:lp_structure_one_resource}, the solution of the LP
relaxation~(\cref{eq:lp_relaxation}) is supported on at most $\min \{ k, m \}$
arms. Informally, our MDP policy,
\texttt{ControlBudget}~(\cref{alg:controlbudget}), samples an arm from a
probability distribution that ensures drifts bounded away from $0$ in the
``correct directions'': (i) a binding resource $j$ has drift at least $\gamma_t$
if $B_{t-1,j} < \tau_t$ and drift at most $-\gamma_t$ if $B_{t-1,j} \geq
\tau_t$; and (ii) a non-binding resource $j$ has drift at least $\frac12
\gamma_t$ if $B_{t-1,j} < \tau_t$.  This allows us to show that the expected
leftover budget for each binding resource and the expected number of null arm
pulls are constants in terms of $T$.

\begin{algorithm}[ht]
  \caption{\texttt{ControlBudget} (for general $m$)}
  \label{alg:controlbudget}
  \DontPrintSemicolon
  \SetKwInOut{Input}{Input}
  \SetKwInOut{Output}{Output}
  \KwIn{time horizon $T$, initial budget $B$, set of arms $\calX$, set of resources $\calJ$, constant $c > 0$.}

  Set $B_{0,j} = B$ for all $j \in \calJ$.

  Define threshold $\tau_t = c \log(T-t)$.

  \For{$t = 1, 2, \dots, T$}{

    \uIf{$\exists j \in \calJ$ such that $B_{t-1,j} < 1$}{
      Pull the null arm $x^0$.
    }
    \Else{
      Define $s_t \in \{ \pm 1 \}^{|X^*|-1} \times 0$ as follows. Let $j$ denote
      the resource corresponding to row $i \in [|X^*| - 1]$ in the matrix $D$
      and vector $b$. Then, the $i$th entry of $s_t$ is $+1$ if $B_{t-1,j} <
      \tau_t$ and $-1$ otherwise.

      Define $\gamma_t$ to be the solution to the following constrained optimization problem:
      \begin{equation}\label{eq:controlbudget_gamma_t}
        \max_{\gamma \in [0,1]} \left\{ \gamma : p = D^{-1} (b + \gamma s_t) \geq 0,\ p^T \mu^{d,j} \geq \frac{\gamma}{2} \ \forall j \in \calJ \setminus J^* \text{ if } B_{t-1,j} < \tau_t  \right\}.
      \end{equation}

      Sample an arm from the probability distribution $p_t = D^{-1}(b + \gamma_t
      s_t)$.
    }

  }
\end{algorithm}

\begin{theorem}\label{theorem:regret_controlbudget}
  If $c \geq \frac{6}{\gamma^{*2}}$, the regret of
  \texttt{ControlBudget}~(\cref{alg:controlbudget}) satisfies
  \begin{equation}
    R_T(\texttt{ControlBudget}) \leq \tilde{C},
  \end{equation}
  where $\gamma^*$ (defined in~\cref{lemma:gammastar}) and $\tilde{C}$ are
  constants with $\tilde{C} = O \left( m \sigma_{\min}^{-1} \left( m
  (\gamma^*)^{-3} \ln \left( \left( (1 - \exp \left( - \gamma^{*2} \right) \right)^{-1} \right)
  + \left( 1 - \exp(\gamma^{*2}) \right)^{-2} \right) \right)$.
\end{theorem}

We defer all proofs in this section
to~\cref{sec:appendix_known_multiple_resources}, but we include a proof sketch
of most results in the main paper following the statement. The proof
of~\cref{theorem:regret_controlbudget} follows from the following sequence of
lemmas. The next two lemmas are generalizations
of~\cref{lemma:regret_one_resource_case_2,lemma:regret_one_resource_case_3} with
essentially the same proofs. Recall $J^*$ denotes the unique set of resources
that correspond to binding constraints in the LP
solution~(\cref{subsec:assumptions}).

\begin{lemma}\label{lemma:regret_multiple_resources_null}
  If the LP solution includes the null arm $x^0$ in its support, then
  \begin{equation}
    R_T(\texttt{ControlBudget}) \leq \tilde{C} \cdot \left( \sum_{j \in J^*} \E [B_{T,j}] \right),
  \end{equation}
  where $\tilde{C} = O( \sigma_{\min}^{-1})$ is a constant.
\end{lemma}

\begin{lemma}\label{lemma:regret_multiple_resources_no_null}
  If the LP solution does not include the null arm $x^0$ in its support, then
  \begin{equation}
    R_T(\texttt{ControlBudget}) \leq \tilde{C} \cdot \left( \sum_{j \in J^*} \E [B_{T,j}] + \E [N_{x^0}] \right),
  \end{equation}
  where $\E [N_{x^0}]$ denotes the expected number of null arm pulls and
  $\tilde{C} = O(\sigma_{\min}^{-1})$ is a constant.
\end{lemma}

\cref{lemma:null_arm_pulls_multiple_resources,lemma:leftover_budget_multiple_resources}
are generalizations
of~\cref{lemma:null_arm_pulls_one_resource,lemma:leftover_budget_one_resource}
with similar proofs after taking a union bound over resources. But we first
need~\cref{lemma:gammastar} that lets us conclude there is drift of magnitude at
least $\gamma^* > 0$ in the ``correct directions'' as stated earlier.

\begin{lemma}\citep[Lemma 14]{flajolet2015logarithmic}\label{lemma:gammastar}
  In each round $t$, $\gamma_t \geq \gamma^* = \frac{\sigma_{\min} \min \{
  \deltasupport, \deltaslack \}}{4m}$.
\end{lemma}
The proof of this lemma is identical to~\citet[Lemma
14]{flajolet2015logarithmic} but we provide a proof in the appendix for
completeness.

\begin{lemma}\label{lemma:null_arm_pulls_multiple_resources}
  If the LP solution does not include the null arm in its support, then
  \begin{equation}
    \E [N_{x^0}] \leq \tilde{C},
  \end{equation}
  where $\tilde{C} = O \left( m (\gamma^*)^{-3} \ln \left( \left( (1 - \exp
  \left( - \frac{\gamma^{*2}}{8} \right) \right)^{-1} \right) \right)$ is a constant.
\end{lemma}

\begin{lemma}\label{lemma:leftover_budget_multiple_resources}
  If the LP solution is supported on more than one arm, then for all $j \in J^*$
  \begin{equation}
    \E [B_{T,j}] \leq \tilde{C},
  \end{equation}
  where $\tilde{C} = \tilde{O} \left( \left( 1 - \exp(\gamma^{*2})
  \right)^{-2} + (\gamma^*)^{-2} \right)$ is a constant.
\end{lemma}

A subtle but important point is that the regret analysis does not require
\texttt{ControlBudget} to know the true expected drifts in order to find the
probability vector $p_t$. It simply requires the algorithm to know $X^*$, $J^*$,
and find any probability vector $p_t$ that ensures drifts bounded away from $0$ 
in the ``correct directions'' as stated earlier. We use this property in our
learning algorithm,
\texttt{ExploreThenControlBudget}~(\cref{alg:explorethencontrolbudget}), in the
next section.

\section{Learning Algorithm with Logarithmic Regret}
\label{sec:unknown_distributions}

In this section we design a learning algorithm,
\texttt{ExploreThenControlBudget}~(\cref{alg:explorethencontrolbudget}), with
logarithmic regret in terms of $T$ for the setting when the learning does not
know the true distributions. Our algorithm, which can be viewed as combining
aspects of~\citet[Algorithm 1]{li2021symmetry} and~\citet[Algorithm
UCB-Simplex]{flajolet2015logarithmic}, proceeds in three phases. It uses phase
one of~\citet[Algorithm 1]{li2021symmetry} to identify the set of optimal arms
$X^*$ and the set of binding constraints $J^*$ by playing arms in a round-robin
fashion, and using confidence intervals and properties of LPs. This is
reminiscent of successive elimination~\citep{evendar2002pac}, except that the
algorithm tries to identify the optimal arms instead of eliminating suboptimal
ones. In the second phase the algorithm continues playing the arms in $X^*$ in a
round-robin fashion to shrink the confidence radius further. In the third phase
the algorithm plays a variant of the MDP policy
\texttt{ControlBudget}~(\cref{alg:controlbudget}) with a slighly different
optimization problem for $\gamma_t$ because it only has empirical esitmates of
the drifts.


\subsection{Additional Notation and Preliminaries}

For all arms $x \in \calX$ and rounds $t \geq k$, define the upper confidence
bound (UCB) of the expected outcome vector $\mu_x^o$ as $\ucb_t(x) =
\bar{o}_t(x) + \rad_t(x)$, where $\rad_t(x) = \sqrt{8 n_t(x)^{-1} \log T}$
denotes the confidence radius, $n_t(x)$ denotes the number of times $x$ has been
played before $t$, and $\bar{o}_t(x) = n_t(x)^{-1} \sum_t o_t \indicator[x_t =
x]$ denotes the empirical mean outcome vector of $x$. The lower confidence bound
(LCB) is defined similarly as $\lcb_t(x) = \bar{o}_t(x) - \rad_t(x)$.

For all arms $x \in \calX$, let $\opt_{-x}$ denote the value of the LP
relaxation~(\cref{eq:lp_relaxation}) with the additional constraint $p_x = 0$,
and for all resources $j \in \calJ$, let $\opt_{-j}$ denote the value when the
objective has an extra $-\sum_{x} p_x \mu_x^{d,j} + \nicefrac{B}{T}$
term~\citep{li2021symmetry}. Intuitively, these represent how important it is to
play arm $x$ or make the resource constraint for $j$ a binding constraint.
Define the UCB of $\opt_{-x}$ to be the value of the LP when the expected
outcome is replaced by its UCB, and denote this by $\ucb_t(\opt_{-x})$. The LCB
for $\opt_{-x}$, and UCB and LCB for $\opt_{-j}$ and $\opt_\lp$ are defined
similarly.

\begin{definition}[Gap~\citep{li2021symmetry}]\label{def:gap}
  The gap of the problem instance is defined as
  \begin{equation}
    \Delta = \min \left\{ \min_{x \in X^*} \left\{ \opt_\lp - \opt_{-x} \right\}, \min_{j \notin J^*} \left\{ \opt_\lp - \opt_{-j} \right\} \right\}.
  \end{equation}
\end{definition}


\subsection{Learning Algorithm and Regret Analysis}

\begin{algorithm}[ht]
  \caption{\texttt{ExploreThenControlBudget}}
  \label{alg:explorethencontrolbudget}
  \DontPrintSemicolon
  \SetKwInOut{Input}{Input}
  \SetKwInOut{Output}{Output}
  \KwIn{time horizon $T$, initial budget $B$, set of arms $\calX$, set of resources $\calJ$, constant $c > 0$.}

  Set $B_{0,j} = B$ for all $j \in \calJ$.

  Initialize $t = 1, X^* = \emptyset, J' = \emptyset$.

  \While{$t < T - k$ and $|X^*| + |J'| < m+1$}{
    Play each arm in $\calX \setminus \{ x^0 \}$ in a round-robin fashion. Play $x^0$ if $\exists j$ such that $B_{t-1,j} < 1$.

    For each $x \in \calX$, if $\ucb_t(\opt_{-x}) < \lcb_t(\opt_\lp)$, then add $x$ to $X^*$.

    For each $j \in \calJ$, if $\ucb_t(\opt_{-j}) < \lcb_t(\opt_\lp)$, then add $j$ to $J'$.
  }

  Set $J^* = \calJ \setminus J'$.

  \While{$t < T - |X^*|$ and $n_t(x) < \frac{32 \log T}{\gamma^{*2}}$ for all $x \in X^*$, where $\gamma^*$ is defined in~\cref{lemma:gammastar}}{
    Play each arm in $X^*$ in a round-robin fashion.
  }

  \While{$t < T$}{
    \uIf{$\exists j \in \calJ$ such that $B_{t-1,j} < 1$}{
      Pull the null arm $x^0$.
    }
    \Else{
    Define $s_t$ as in~\cref{alg:controlbudget}.

    Choose $(\gamma_t, p_t) = \max_{\gamma \in [0,1]} \gamma$ such that there exists a probability vector $p$ satisfying \label{algline:opt_prob}
    \begin{align}
      \lcb_t(p^T \mu^{d,j}) &\geq \frac{\gamma}{8} \  \forall j \in \calJ \setminus J^* \text{ if } B_{t-1,j} < \tau_t, \label{eq:etcb_gamma_t_1} \\
      \lcb_t(p^T \mu^{d,j}) &\geq \frac{\gamma}{8} \  \forall j \in J^* \text{ if } B_{t-1,j} < \tau_t \label{eq:etcb_gamma_t_2} \\
      \ucb_t(p^T \mu^{d,j}) &\leq - \frac{\gamma}{8} \  \forall j \in J^* \text{ if } B_{t-1,j} \geq \tau_t. \label{eq:etcb_gamma_t_3}
    \end{align}

    Sample an arm from the probability distribution $p_t$.
    }
  }

\end{algorithm}

\begin{theorem}\label{theorem:regret_explorethencontrolbudget}
  If $c \geq \frac{6}{\gamma^*}$, the regret of
  \texttt{ExploreThenControlBudget}~(\cref{alg:explorethencontrolbudget})
  satisfies
  \begin{equation}
    R_T(\texttt{ExploreThenControlBudget}) \leq \tilde{C} \cdot \log T,
  \end{equation}
  where $\gamma^*$ (defined in~\cref{lemma:gammastar}), $\tilde{C}'$ and
  $\tilde{C}$ are constants with $\tilde{C}'$ denoting the constant
  in~\cref{theorem:regret_controlbudget} and \begin{equation} \tilde{C} = O
  \left( \frac{k m^2}{\min \{ \deltadrift^2, \sigma_{\min}^2 \} \Delta^2} + k
  (\gamma^*)^{-2} + \tilde{C}' \right).
  \end{equation}
\end{theorem}
We defer all proofs in this section to~\cref{sec:appendix_unknown}, but we
include a proof sketch of most results in the main paper following the 
statement.  We refer to the three while loops of
\texttt{ExploreThenControlBudget}~(\cref{alg:explorethencontrolbudget}) as the
three phases. The lemmas and corollaries following the theorem below show that
the first phase consists of at most logarithmic number of rounds. It is easy to
see that the second phase consists of at most logarithmic number of rounds. The
third phase plays a variant of the MDP policy
\texttt{ControlBudget}~(\cref{alg:controlbudget}). There exists a feasible
solution to the optimization problem in~\cref{algline:opt_prob} that ensures
drifts bounded away from $0$ in the ``correct
directions''~(\cref{subsec:appendix_theorem_regret_explorethencontrolbudget}).
Combining the analysis of the first two phases with
~\cref{theorem:regret_controlbudget} lets us conclude that
\texttt{ExploreThenControlBudget} has logarithmic regret.

\begin{definition}[Clean Event]\label{def:clean_event}
  The clean event is the event such that for all $x \in \calX$ and all $t \geq
  k$, (i) $\mu_x^0 \in [\lcb_t(x), \ucb_t(x)]$; and (ii) after the first $n$
  pulls of the null arm the sum of the drifts for each resource is at least $w$,
  where $w = \frac{4096 k m^2 \log T}{\deltadrift^2 \Delta^2}$ and $n = \frac{2
  w}{\mu_{x^0}^d}$.
\end{definition}

\begin{lemma}\label{lemma:learning_clean_event}
  The clean event occurs with probability at least $1 - 5 m T^{-2}$.
\end{lemma}
The lemma follows from the Azuma-Hoeffding inequality. Since the complement of the clean
event contributes $O(m T^{-1})$ to the regret, it suffices to bound the regret
conditioned on the clean event.

\begin{lemma}\label{lemma:learning_lp_confidence_interval}
  If the clean event occurs, then $\ucb_t(\opt_\lp) - \lcb_t(\opt_\lp) \leq
  \frac{8 m}{\sigma_{\min}} \rad_t$. A similar statement is true for $\opt_{-x}$
  and $\opt_{-j}$ for all $x \in \calX$ and $j \in \calJ$.
\end{lemma}
The proof follows from a perturbation analysis of the LP and uses the confidence
radius to bound the perturbations in the rewards and drifts.

\begin{corollary}
  If the clean event occurs and $n_t(x) > \frac{2048 m^2 \log T}{\sigma_{\min}^2
  \Delta^2}$ for all $x \in \calX$, then
  \begin{equation}
    \ucb_t(\opt_{-x}) < \lcb_t(\opt_\lp) \text{ and } \ucb_t(\opt_{-j}) < \lcb_t(\opt_\lp)
  \end{equation}
  for all $x \in X^*$ and $j \in \calJ \setminus J^*$.
\end{corollary}
This follows from substituting the bound on $n_t(x)$ into the definition of
$\rad_t(x)$ and applying~\cref{lemma:learning_lp_confidence_interval}. In the
worst case, each pull of an arm can cause the budget to drop below $1$, but the
clean event implies that the first $n$ pulls of $x^0$ have enough total drift to
allow $n_t(x)$ pulls of each non-null arm in phase 1. This allows us to upper
bound the duration of phase 1 as follows.

\begin{corollary}\label{corollary:learning_phase_one_rounds}
  If the clean event occurs, then phase 1 of \texttt{ExploreThenControlBudget}
  has at most
  \begin{equation}
    \tilde{C} \cdot \log T
  \end{equation}
  rounds, where $\tilde{C} = O \left( \frac{k m^2}{\min \{ \deltadrift^2,
  \sigma_{\min}^2 \} \Delta^2} \right)$ is a constant.
\end{corollary}


\subsection{Reduction from $\bwk$}

\begin{theorem}\label{theorem:reduction}
  Suppose $\frac{B}{T} \geq \deltadrift$. Consider a $\bwk$ instance such that (i) for each arm and resource, the expected consumption of that resource differs from $\frac{B}{T}$ by at least $\deltadrift$; and (ii) all the other assumptions required by~\cref{theorem:regret_explorethencontrolbudget}~(\cref{subsec:assumptions}) are also satisfied. Then, there is an algorithm for $\bwk$ whose regret satisfies the same bound as in~\cref{theorem:regret_explorethencontrolbudget} with the same constant $\tilde{C}$.
\end{theorem}
Due to space constraints, we present the reduction in
~\cref{subsec:appendix_reduction}.

\section{Conclusion}
\label{sec:conclusion}

In this paper we introduced a natural generalization of $\bwk$ that allows
non-monotonic resource utilization. We first considered the setting when the
decision-maker knows the true distributions and presented an MDP policy with
\emph{constant} regret against an LP relaxation. Then we considered the setting
when the decision-maker does not know the true distributions and presented a
learning algorithm with \emph{logarithmic} regret against the same LP
relaxation. Finally, we also presented a reduction from $\bwk$ to our model and
showed a regret bound that matches existing results~\citep{li2021symmetry}.

An important direction for future research is to obtain optimal regret bounds.
The regret bound for our algorithm scales as $O(\poly (k) \poly(m)
\Delta^{-2})$, where $k$ is the number of arms, $m$ is the number of resources
and $\Delta$ is the suboptimality parameter. A modification to our algorithm
along the lines of~\citet[algorithm UCB-Simplex]{flajolet2015logarithmic} that
considers each support set of the LP solution explicitly leads to a regret bound
that scales as $O(\poly (k) 2^m \Delta^{-1})$. It is an open question, even for
$\bwk$, to obtain a regret bound that scales as $O(\poly (k) \poly(m)
\Delta^{-1})$ or show that the trade-off between the dependence on the number of
resources and the suboptimality parameter is unavoidable. 

Another natural follow-up to our work is to develop further extensions, such as
considering an infinte set of arms~\citep{kleinberg2019bandits}, studying
adversarial observations~\citep{auer2002nonstochastic,immorlica2019adversarial},
or incorporating contextual
information~\citep{slivkins2011contextual,agrawal2016linear} as has been the
case elsewhere throughout the literature on bandits.

\begin{ack}
  We thank Kate Donahue, Sloan Nietert, Ayush Sekhari, and Alex Slivkins for
  helpful discussions. This research was partially supported by the NSERC
  Postgraduate Scholarships-Doctoral Fellowship 545847-2020.
\end{ack}

\bibliographystyle{plainnat}
\bibliography{refs}

\section*{Checklist}


\begin{enumerate}

\item For all authors...
\begin{enumerate}
  \item Do the main claims made in the abstract and introduction accurately reflect the paper's contributions and scope?
    \answerYes{}
  \item Did you describe the limitations of your work?
    \answerYes{We discuss potential improvements in~\cref{sec:conclusion}.}
  \item Did you discuss any potential negative societal impacts of your work?
    \answerNo{Our paper studies a theoretical model and we do not believe there are any negative societal impacts of this paper.}
  \item Have you read the ethics review guidelines and ensured that your paper conforms to them?
    \answerYes{}
\end{enumerate}

\item If you are including theoretical results...
\begin{enumerate}
  \item Did you state the full set of assumptions of all theoretical results?
    \answerYes{We state our assumptions in~\cref{subsec:assumptions}.}
  \item Did you include complete proofs of all theoretical results?
    \answerYes{We include complete proofs in the appendix. But we include a proof sketch of most lemmas in the main paper following the lemma statement.}
\end{enumerate}

\item If you ran experiments...
\begin{enumerate}
  \item Did you include the code, data, and instructions needed to reproduce the main experimental results (either in the supplemental material or as a URL)?
    \answerNA{}
  \item Did you specify all the training details (e.g., data splits, hyperparameters, how they were chosen)?
    \answerNA{}
  \item Did you report error bars (e.g., with respect to the random seed after running experiments multiple times)?
    \answerNA{}
  \item Did you include the total amount of compute and the type of resources used (e.g., type of GPUs, internal cluster, or cloud provider)?
    \answerNA{}
\end{enumerate}

\item If you are using existing assets (e.g., code, data, models) or curating/releasing new assets...
\begin{enumerate}
  \item If your work uses existing assets, did you cite the creators?
    \answerNA{}
  \item Did you mention the license of the assets?
    \answerNA{}
  \item Did you include any new assets either in the supplemental material or as a URL?
    \answerNA{}
  \item Did you discuss whether and how consent was obtained from people whose data you're using/curating?
    \answerNA{}
  \item Did you discuss whether the data you are using/curating contains personally identifiable information or offensive content?
    \answerNA{}
\end{enumerate}

\item If you used crowdsourcing or conducted research with human subjects...
\begin{enumerate}
  \item Did you include the full text of instructions given to participants and screenshots, if applicable?
    \answerNA{}
  \item Did you describe any potential participant risks, with links to Institutional Review Board (IRB) approvals, if applicable?
    \answerNA{}
  \item Did you include the estimated hourly wage paid to participants and the total amount spent on participant compensation?
    \answerNA{}
\end{enumerate}

\end{enumerate}


\newpage
\appendix

\section{Assumptions about the Null Arm}
\label{sec:appendix_assumptions}

In any model in which resources can be consumed and/or replenished over time,
one must specify what happens when the budget of one (or more) resources reaches
zero. The original bandits with knapsacks problem assumes than when this
happens, the process of learning and gaining rewards ceases. The key distinction
between that model and ours is that we instead assume the learner is allowed
remain idle until the supply of every resource becomes positive again, at which
point the learning process recommences. The null arm in our paper is intended to
represent this option to remain idle and wait for resource replenishment. In
order for these idle periods to have finite length almost surely, a minimal
assumption is that when the null arm is pulled, for each resource there is a
positive probability that the supply of the resource increases. We make the
stronger assumption that for each resource, the expected change in supply is
positive when the null arm is pulled. In fact, our results for the MDP setting
hold under the following more general assumption: there exists a probability
distribution over arms, such that when a random arm is sampled from this
distribution and pulled, the expected change in the supply of each resource is
positive. In the following, we refer to this as Assumption PD (for "positive
drift").

To see that our results for the MDP setting continue to hold under Assumption PD
(i.e., even if one doesn't assume that the null arm itself is guaranteed to
yield positive expected drift for each resource) simply
modify~\cref{alg:controlbudget,alg:explorethencontrolbudget} so that whenever
they pull the null arm in a time step when the supply of each resource is at
least 1, the modified algorithms instead pull a random arm sampled from the
probability distribution over arms that guarantees positive expected drift for
every resource. As long as the constant $\deltadrift$ is less than or equal to
this positive expected drift, the modification to the algorithms does not change
their analysis. We believe it's likely that our learning
algorithm~(\cref{alg:explorethencontrolbudget}) could similarly be adapted to
work under Assumption PD, but it would be less straightforward because the
positive-drift distribution over arms would need to be learned.

When Assumption PD is violated, the problem becomes much more similar to the
Bandits with Knapsacks problem. To see why, consider a two-player zero-sum game
in which the row player chooses an arm $x$, the column player chooses a resource
$j$, and the payoff is the expected drift of that resource when that arm is
pulled, $\mu_{x}^{d,j}$. Assumption PD is equivalent to the assertion that the
value of the game is positive; the negation of Assumption PD means that the
value of the game is negative. By the Minimax Theorem, this means there is a
convex combination of resources (i.e., a mixed strategy for the column player)
such that the weighted-average supply of these resources is guaranteed to
experience non-positive expected drift, no matter which arm is pulled. Either
the expected drift is zero --- we prove in~\cref{sec:appendix_assumptions} of
the supplementary material that regret $O(\sqrt{T})$ is unavoidable in this case
--- or the expected drift is strictly negative, in which case the
weighted-average resource supply inevitably dwindles to zero no matter which
arms the learner pulls. In either case, the behavior of the model is
qualitatively different when Assumption PD does not hold.
\section{Regret Bounds for One Arm, One Resource, and Zero Drift}
\label{sec:appendix_zero_drift}

In this section we will consider the case when $\calX = \{ x^0, x \}$, $\calJ =
\{ 1 \}$, and $x$ has zero drift, i.e., $\mu_x^d = 0$. Since $x$ is the only arm
besides the null arm, we assume without loss of generality that its reward is
equal to $1$ deterministically. The optimal policy is to pull $x^0$ when
$B_{t-1} < 1$ and $x$ otherwise. We will show that the regret of this policy is
$\Theta(\sqrt{T})$.

\begin{theorem}\label{theorem:one_arm_zero_drift_regret_upper}
  The regret of the MDP policy is $O(\sqrt{T})$.
\end{theorem}

\begin{proof}
  The optimal solution of the LP relaxation~(\cref{eq:lp_relaxation}) is $p_x =
  1$ and $p_{x^0} = 0$. Since $x^0$ and $x$ have reward equal to $0$ and $1$
  deterministically, $\opt_\lp = 1$.Therefore, the regret of the MDP policy is
  equal to the expected number rounds in which the budget is less than $1$. That
  is,
  \begin{equation*}
    R_T = \E \left[ \sum_{t=1}^T \indicator[B_{t-1} < 1] \right].
  \end{equation*}
  Since $\E[d_t] = 0$ when $B_{t-1} \geq 1$ and $\E[d_t] = \mu_{x^0}^d$ when
  $B_{t-1} < 1$, we can write
  \begin{align*}
    R_T &= \E \left[ \sum_{t=1}^T \indicator[B_{t-1} < 1] \right] \\
    &= \frac{1}{\mu_{x^0}^d} \mu_{x^0}^d \E \left[ \sum_{t=1}^T \indicator[B_{t-1} < 1] \right] \\
    &= \frac{1}{\mu_{x^0}^d} \E \left[ \sum_{t=1}^T \mu_{x^0}^d \indicator[B_{t-1} < 1] + 0 \indicator[B_{t-1} \geq 1] \right] \\
    &= \frac{1}{\mu_{x^0}^d} \E \left[ \sum_{t=1}^T d_t \right] \\
    &= \frac{1}{\mu_{x^0}^d} \left( \E \left[ B_{T} \right] - B \right).
  \end{align*}
  Since $B_0 = B$, the budget is updated as $B_t = B_{t-1} + d_t$ and $d_t \in
  [-1, 1]$, we have
  \begin{align*}
    \E \left[ B_t^2 \vert B_{t-1} \right] &= \E \left[ B_{t-1}^2 + 2 B_{t-1} d_t + d_t^2 \vert B_{t-1} \right] \\
    &= \E \left[ B_{t-1}^2 \vert B_{t-1} \right] + \E \left[ 2 B_{t-1} d_t \vert B_{t-1} \right] + \E \left[ d_t^2 \vert B_{t-1} \right] \\
    &\leq B_{t-1}^2 + \E \left[ 2 B_{t-1} d_t \vert B_{t-1} \right] + 1^2 \\
    &= B_{t-1}^2 + 2 B_{t-1} \mu_{x^0}^d \indicator[B_{t-1} < 1] + 1 \\
    &\leq B_{t-1}^2 + 2 \mu_{x^0}^d + 1 \\
    \Rightarrow \E \left[ B_T^2 \right] &= O(T).
  \end{align*}
  Using Jensen's inequality, we have
  \begin{align*}
    \E \left[ B_T \right] \leq \sqrt{\E \left[ B_T^2 \right]} &= O(\sqrt{T}).
  \end{align*}
  This completes the proof.
\end{proof}

\begin{theorem}\label{theorem:one_arm_zero_drift_regret_lower}
  If $\E \left[ d_t^2 \vert B_{t-1} \right] \geq \sigma^2 > 0$, then the regret
  of the MDP policy is $\Omega(\sqrt{T})$.
\end{theorem}

\begin{proof}
  Using the proof of~\cref{theorem:one_arm_zero_drift_regret_upper}, it suffices
  to provide a lower bound on $\E \left[ B_T \right]$. Since the budget is
  updated as $B_t = B_{t-1} + d_t$, $\E \left[ d_t \vert B_{t-1} \right] \geq
  0$, and $\E \left[ d_t^2 \vert B_{t-1} \right] \geq \sigma^2$, we have
  \begin{align*}
    \E \left[ B_t^2 \vert B_{t-1} \right] &= \E \left[ B_{t-1}^2 + 2 B_{t-1} d_t + d_t^2 \vert B_{t-1} \right] \\
    &= \E \left[ B_{t-1}^2 \vert B_{t-1} \right] + \E \left[ 2 B_{t-1} d_t \vert B_{t-1} \right] + \E \left[ d_t^2 \vert B_{t-1} \right] \\
    &\geq B_{t-1}^2 + \E \left[ 2 B_{t-1} d_t \vert B_{t-1} \right] + \sigma^2 \\
    &= B_{t-1}^2 + 2 B_{t-1} \E \left[d_t \vert B_{t-1} \right] + \sigma^2 \\
    &\geq B_{t-1}^2 + \sigma^2 \\
    \Rightarrow \E \left[ B_T^2 \right] &\geq \Omega(T). 
  \end{align*}
  The Cauchy-Schwarz inequality yields that
  \begin{equation*}
    \E \left[ \left( B_T^{\nicefrac{1}{2}} \right)^2 \right]^{\nicefrac{1}{2}} \E \left[ \left( B_T^{\nicefrac{3}{2}} \right)^2 \right]^{\nicefrac{1}{2}} \geq \E \left[ B_T^2 \right] \geq \Omega(T). 
  \end{equation*}
  Squaring both sides yields that
  \begin{equation*}
    \E \left[ B_T \right] \E \left[ B_T^3 \right] \geq \E \left[ B_T^2 \right]^2 \geq \Omega(T^2). 
  \end{equation*}
  It suffices to show that $\E \left[ B_T^3 \right] = O(T^{\nicefrac{3}{2}})$
  because this will imply that $\E \left[ B_T \right] =
  \Omega(T^{\nicefrac{1}{2}})$. Since $d_t \in [-1, 1]$, we have
  \begin{align*}
    & \E \left[ B_t^3 \vert B_{t-1} \right] \\
    &= \E \left[ B_{t-1}^3 + 3 B_{t-1}^2 d_t + 3 B_{t-1} d_t^2 + d_t^3 \vert B_{t-1} \right] \\
    &= B_{t-1}^3 + 3 B_{t-1}^2 \E \left[ d_t \vert B_{t-1} \right] + 3 B_{t-1} \E \left[ d_t^2 \vert B_{t-1} \right] + \E \left[ d_t^3 \vert B_{t-1} \right] \\
    &= B_{t-1}^3 + 3 B_{t-1}^2 \mu_{x^0}^d \indicator[B_{t-1} < 1] + 3 B_{t-1} \E \left[ d_t^2 \vert B_{t-1} \right] + \E \left[ d_t^3 \vert B_{t-1} \right] \\
    &\leq B_{t-1}^3 + 3 \mu_{x^0}^d + 3 B_{t-1} + 1.
  \end{align*}
  Taking expectation on both sides yields
  \begin{align*}
    \E \left[ B_t^3 \right] & \leq \E \left[ B_{t-1}^3 \right] + 3 \E \left[ B_{t-1} \right] + O(1) \\
    &\leq \E \left[ B_{t-1}^3 \right] + O(\sqrt{t}),
  \end{align*}
  where the last inequality follows from the proof
  of~\cref{theorem:one_arm_zero_drift_regret_upper} where we show that $\E
  \left[ B_t \right] \leq O(\sqrt{t})$. Summing both sides over all rounds
  yields that
  \begin{equation*}
    \E \left[ B_T^3 \right] = O(T^{\nicefrac{3}{2}})
  \end{equation*}
  and this completes the proof.
\end{proof}
\section{Proofs for~\cref{subsec:policy_one_resource}}
\label{sec:appendix_known_one_resource}


\subsection{Proof of~\cref{lemma:one_arm_pos_drift}}

When the LP solution is supported on a positive drift arm $x^p$, $\opt_\lp = 1$
because the LP plays it with probability $1$. Therefore, the regret is equal to
the expected number of times
\texttt{ControlBudget}~(\cref{alg:controlbudget_one_resource}) pulls the null
arm. This, in turn, is equal to the expected number of rounds in which the
budget is less than $1$.

Define
\begin{equation}
  b_0 = 8 \deltadrift^{-2} \ln \left( \frac{2}{1 - \exp \left( - \frac{\deltadrift^2}{8} \right)} \right).
\end{equation}
Then, we have that for all $b \geq b_0$,
\begin{align}
  \sum_{k=b}^\infty \Pr \left[ B_{s+k} \in [0, 1) \vert B_s = b \right] &\leq \sum_{k=b}^\infty \exp \left( - \frac{\deltadrift^2 k}{8} \right) \\
  &= \exp \left( - \frac{\deltadrift^2 b}{8} \right) \left( 1 - \exp \left( - \frac{\deltadrift^2}{8} \right) \right)^{-1}.
\end{align}
where the first inequality follows from Azuma-Hoeffding's inequality. By our
choice of $b_0$, we have that
\begin{equation}\label{eq:one_arm_pos_drift_1}
  \sum_{k=b}^\infty \Pr \left[ B_{s+k} \in [0, 1) \vert B_s = b \right] \leq \frac12.
\end{equation}
In words, the probability that the budget ever drops below $1$ once it exceeds
$b_0$ is at most $\frac12$. Now, consider the following recursive definition for
two disjoint sequence of indices $s_i$ and $s_i'$. Let $s_0 = \min \{ t \geq 1 :
B_{t-1} \in [0, 1) \}$, and define
\begin{align}
  s_i' &= \min \{ t > s_i : B_{t-1} \geq b_0 \text{ or } t-1 = T \} \\
  s_{i+1} &= \min \{ t > s_i' : B_{t-1} \in [0, 1)  \}.
\end{align}
In words, $s_i'$ denotes the first round after $s_i$ in which the budget is at
least $b_0$ and $s_{i+1}$ denotes the first round after $s_i'$ in which the
budget is less than $1$. Note that~\cref{eq:one_arm_pos_drift_1} implies that
\begin{equation}
  \Pr \left[ s_i \text{ is defined } \vert s_{i-1}' \text{ is defined} \right] \leq \frac12.
\end{equation}
Therefore,
\begin{equation}
  \Pr \left[ s_i \text{ is defined} \right] \leq \prod_{j=1}^{i} \Pr \left[ s_j \text{ is defined } \vert s_{j-1}' \text{ is defined} \right] \leq \frac{1}{2^i}.
\end{equation}
Now, we can upper bound the expected number of rounds in which the budget is
below $1$ as
\begin{align}
  \E \left[ \sum_{t=1}^T \indicator[B_{t-1} < 1] \right] &= \sum_{i=0}^{T-1} \Pr \left[ s_i \text{ is defined} \right] \E \left[ \sum_{t=s_i}^{s_i'} \indicator[B_{t-1} < 1] \right] \\
  &\leq \sum_{t=0}^{T-1} 2^{-i} \E \left[ \sum_{t=s_i}^{s_i'} \indicator[B_{t-1} < 1] \right] \\
  &\leq \sum_{t=0}^{T-1} 2^{-i} \E \left[ s_i' - s_i \right] \\
  &\leq \sum_{t=0}^{T-1} 2^{-i} \frac{1}{\deltadrift} \E \left[ B_{s_i'} - B_{s_i} \right] \label{eq:one_arm_pos_drift_main_align_1} \\
  &\leq \sum_{t=0}^{T-1} 2^{-i} \frac{1}{\deltadrift} (b_0 + 1) \\
  &\leq 2 \frac{b_0 + 1}{\deltadrift},
\end{align}
where~\cref{eq:one_arm_pos_drift_main_align_1} follows because both the null arm
and the positive drift arm have drift at least $\deltadrift$. Therefore, we have
that
\begin{equation}
  \E \left[ \sum_{t=1}^T \indicator[B_{t-1} < 1] \right] \leq \tilde{C},
\end{equation}
where
\begin{equation}\label{eq:one_arm_pos_drift_tildec}
  \tilde{C} = O \left( \deltadrift^{-3} \ln \left( \frac{2}{1 - \exp \left( -
  \frac{\deltadrift^2}{8} \right)} \right) \right).
\end{equation}

\subsection{Proof of~\cref{lemma:regret_one_resource_case_2}}

Let $p^*$ denote the optimal solution to the LP relaxation and note that
$Tp^*_x$ denotes the expected number of times the LP plays arm $x$. Since the LP
solution is supported on two arms, both the budget and sum-to-one constraints
are tight. Therefore, we have
\begin{equation}\label{eq:lemma_regret_one_resource_case_2_lp_plays}
  D(Tp^*) = b_\lp,
\end{equation}
where
\begin{equation}
  D = \begin{bmatrix} \mu_{x^0}^d & \mu_{x^n}^d \\ 1 & 1 \end{bmatrix},
  \ 
  p^* = \begin{bmatrix} p^*_{x^0} \\ p^*_{x^n} \end{bmatrix},
  \ 
  b_\lp = \begin{bmatrix} -B \\ T \end{bmatrix}.
\end{equation}

Let $N_x$ denote the number of times
\texttt{ControlBudget}~(\cref{alg:controlbudget_one_resource}) plays arm $x$.
Since it plays the null arm $x^0$ and the negative drift arm $x^n$, the
sum-to-one constraint is tight. However, the budget constraint may not be tight
because there may be leftover budget. Therefore, we have
\begin{equation}\label{eq:lemma_regret_one_resource_case_2_mdp_plays}
  D N = b_\lp - b,
\end{equation}
where
\begin{equation}
  N = \begin{bmatrix} \E[N_{x^0}] \\ \E[N_{x^n}] \end{bmatrix},
  \ 
  b = \begin{bmatrix} -E[B_T] \\ 0 \end{bmatrix}.
\end{equation}

Define
\begin{equation}
  \xi = \begin{bmatrix} \xi_{x^0} \\ \xi_{x^n} \end{bmatrix} = \begin{bmatrix} Tp^*_{x^0} - \E[N_{x^0}] \\ Tp^*_{x^n} - \E[N_{x^n}] \end{bmatrix}.
\end{equation}
Subtracting~\cref{eq:lemma_regret_one_resource_case_2_mdp_plays}
from~\cref{eq:lemma_regret_one_resource_case_2_lp_plays} we have $\xi = D^{-1}
b$, where the LP constraint matrix $D$ is invertible by our assumption that the
drifts are nonzero. Finally, letting $\mu^r$ denote the vector of expected
rewards, the regret can be expressed as
\begin{align}
  R_T(\texttt{ControlBudget}) &= \xi^T \mu^r \\
  &\leq |\xi^T \mu^r| \\
  &\leq \|\xi\|_1 \|\mu^r\|_\infty \\
  &\leq \|D^{-1}\|_1 \|b\|_1 \\
  &\leq C_{\deltadrift} \E[B_T],
\end{align}
where $C_{\deltadrift} = O(\deltadrift^{-1})$ is a constant. This completes the
proof.


\subsection{Proof of~\cref{lemma:regret_one_resource_case_3}}

Let $p^*$ denote the optimal solution to the LP relaxation and note that $Tp^*_x$ denotes the expected number of times the LP plays arm $x$. Since the LP solution is supported on two arms, both the budget and sum-to-one constraints are tight. Therefore, we have
\begin{equation}\label{eq:lemma_regret_one_resource_case_3_lp_plays}
  D(Tp^*) = b_\lp,
\end{equation}
where
\begin{equation}
  D = \begin{bmatrix} \mu_{x^p}^d & \mu_{x^n}^d \\ 1 & 1 \end{bmatrix},
  \ 
  p^* = \begin{bmatrix} p^*_{x^p} \\ p^*_{x^n} \end{bmatrix},
  \ 
  b_\lp = \begin{bmatrix} -B \\ T \end{bmatrix}.
\end{equation}

Let $N_x$ denote the number of times \texttt{ControlBudget}~(\cref{alg:controlbudget_one_resource}) plays arm $x$. Since it plays the null arm $x^0$ when the budget is less than $1$ and may have leftover budget, neither the budget nor the sum-to-one constraint are tight. Therefore, we have
\begin{equation}\label{eq:lemma_regret_one_resource_case_3_mdp_plays}
  D N = b_\lp - b,
\end{equation}
where
\begin{equation}
  N = \begin{bmatrix} \E[N_{x^p}] \\ \E[N_{x^n}] \end{bmatrix},
  \ 
  b = \begin{bmatrix} -E[B_T] \\ \E[N_{x^0}] \end{bmatrix}.
\end{equation}

Define 
\begin{equation}
  \xi = \begin{bmatrix} \xi_{x^p} \\ \xi_{x^n} \end{bmatrix} = \begin{bmatrix} Tp^*_{x^p} - \E[N_{x^p}] \\ Tp^*_{x^n} - \E[N_{x^n}] \end{bmatrix}.
\end{equation}
Subtracting~\cref{eq:lemma_regret_one_resource_case_3_mdp_plays}
from~\cref{eq:lemma_regret_one_resource_case_3_lp_plays} we have $\xi = D^{-1}
b$, where the LP constraint matrix $D$ is invertible by our assumption that the
drifts are nonzero. Finally, letting $\mu^r$ denote the vector of expected
rewards, the regret can be expressed as
\begin{align}
  R_T(\texttt{ControlBudget}) &= \xi^T \mu^r \\
  &\leq |\xi^T \mu^r| \\
  &\leq \|\xi\|_1 \|\mu^r\|_\infty \\
  &\leq \|D^{-1}\|_1 \|b\|_1 \\
  &\leq C_{\deltadrift} \left( \E[B_T] + \E[N_{x^0}] \right),
\end{align}
where $C_{\deltadrift} = O(\deltadrift^{-1})$ is a constant. This completes the
proof.


\subsection{Proof of~\cref{lemma:null_arm_pulls_one_resource}}
\label{subsec:appendix_lemma_null_arm_pulls_one_resource}

Divide the $T$ rounds into two phases: $P_1 = \{ 1, \dots, T -
\exp(\nicefrac{3}{c}) \}$ and $P_2 = \{ 1, \dots, T \} \setminus P_1$. Note that
$P_2$ consists of $\exp(\nicefrac{3}{c}) = O(\exp(\deltadrift)) = O(1)$ rounds,
where the last equality follows because drifts are bounded by $1$. Therefore,
the expected number of null arm pulls in this phase is $O(1)$ and it suffices to
bound the expected number of null arm pulls in $P_1$.

Consider the following recursive definition for three disjoint sequences of
indices $t_i, t_i'$ and $t_i''$. Let $t_0 = 0$, and define
\begin{align}
  t_i' &= \min \{ t > t_i : B_{t-1} \geq \tau_t \text{ or } t-1 = T \}, \\
  t_i'' &= \min \{ t > t_i' : B_{t-1} < \tau_t \}, \\
  t_{i+1} &= \min \{ t > t_i'' : B_t < 1 \}.
\end{align}

We can bound the expected number of rounds in which the budget is less than $1$
as
\begin{align}
  &\E \left[ \sum_{t=1}^T \indicator[B_{t-1} < 1] \right] \\
  &= \sum_{i=0}^{T-1} \Pr \left[ t_{i} \text{ exists } \right] \E \left[ \sum_{t=t_i}^{t_i'} \indicator[B_{t-1} < 1] \right] \\
  &\leq \underbrace{\E \left[ \sum_{t=t_0}^{t_0' - 1} \indicator[B_{t-1} < 1] \right]}_{(a)} + \sum_{i=0}^{T-1} \Pr \left[ t_{i+1} \text{ exists } \vert t_i', t_i'' \text{ exist} \right] \underbrace{\E \left[ \sum_{t=t_i}^{t_i' - 1} \indicator[B_{t-1} < 1] \right]}_{(a)}.
\end{align}
In rounds $\{ t_i, \dots, t_i' - 1 \}$, the algorithm pulls the null and
positive drift arms. The proof of~\cref{lemma:one_arm_pos_drift} shows that the
expected number of null arm pulls in these rounds is at most $\tilde{C}$, where
$\tilde{C}$ is defined in~\cref{eq:one_arm_pos_drift_tildec}. Therefore, we can
bound the term (a) in the above inequality by $\tilde{C}$ and we have that
\begin{equation}\label{eq:leftover_budget_one_resource_main_inequality}
  \E \left[ \sum_{t=1}^T \indicator[B_{t-1} < 1] \right] \leq \tilde{C} \left( 1 +  \sum_{i=0}^{T-1} \Pr \left[ t_{i+1} \text{ exists } \vert t_i', t_i'' \text{ exist} \right] \right).
\end{equation}
If $t_i'$ exists, then $B_{t_i' - 1} \geq \tau_{t_i'}$. If $t_i''$ exists, then
$\tau_{t_i''} - 1 \leq B_{t_i'' - 1} < \tau_{t_i''}$ because (i) $t_i''$ is the
first round after $t_i'$ in which the budget is below the threshold; and (ii)
the drifts are bounded by $1$, so it cannot be lower than $\tau_{t_i''} - 1$.
The algorithm pulls the negative drift arm $x^n$ in the rounds $\{ t_i',
\dots, t_i'' - 1 \}$ and the positive drift arm $x^p$ in the rounds $\{ t_i'',
\dots, t_{i+1}-1 \}$. Since the drifts are bounded by $1$, it takes at least
$\tau_{t_i''}-2$ rounds for the budget to drop below $1$ after repeated pulls of
$x^p$. Using this and the observation that the budget dropping below $1$ is
contained in the event that the total drift in those rounds is nonpositive, we
can bound (a) as
\begin{align}
  \Pr \left[ t_{i+1} \text{ exists } \vert t_i'', t_i' \text{ exist} \right] &\leq \sum_{q=t_i'' + \tau_{t_i''} - 2}^T \Pr \left[ \sum_{t=t_i''+1}^q d_t \leq 0 \right] \\
  &\leq \sum_{q=t_i'' + \tau_{t_i''} - 2}^T \exp \left( - \frac12 \deltadrift^2 (\tau_{t_i''} - 2) \right) \\
  &\leq \sum_{q=t_i'' + \tau_{t_i''} - 2}^T \exp \left( - \frac12 \deltadrift^2 \tau_{t_i''} \right) \\
  &= \sum_{q=t_i'' + \tau_{t_i''} - 2}^T \exp \left( - \frac12 \deltadrift^2 c \log (T - t_i'') \right) \\
  &\leq \sum_{q=t_i'' + \tau_{t_i''} - 2}^T (T - t_i'')^{-3},
\end{align}
where the second inequality follows from the Azuma-Hoeffding inequality applied
to the sequence of drifts sampled from $x^p$ and the last inequality follows
because $c \geq \frac{6}{\deltadrift^2}$. The summation is over at most $T -
t_i''$ terms because there are at most $T - t_i''$ rounds left after round
$t_i''$. Therefore, we have that
\begin{equation}
  \Pr \left[ t_{i+1} \text{ exists } \vert t_i'', t_i' \text{ exist} \right] \leq (T - t_i'')^{-2}.
\end{equation}
Substituting this in~\cref{eq:leftover_budget_one_resource_main_inequality}, we have that
\begin{align}
  \E \left[ \sum_{t=1}^T \indicator[B_{t-1} < 1] \right] &\leq \tilde{C} \left( 1 + \sum_{i=0}^{T-1} \Pr \left[ t_{i+1} \text{ exists } \vert t_i', t_i'' \text{ exist} \right] \right) \\
  &\leq \tilde{C} \left( 1 + \sum_{i=0}^{T-1} (T - t_i'')^{-2} \right) \\
  &\leq \tilde{C} \left( 1 + \sum_{i=0}^{\infty} (T - t_i'')^{-2} \right) \\
  &\leq \tilde{C} \left( 1 + \frac{\pi^2}{6} \right).
\end{align}
This completes the proof.

\subsection{Proof of~\cref{lemma:leftover_budget_one_resource}}

Let $E_q$ denote the event that the negative drift arm $x^n$ is pulled
consecutively in exactly the last $q$ rounds, i.e., $x_t = x^n$ for all $t \geq
T-q+1$ and $x_t \in \{ x^0, x^p \}$ for $t = T-q$ (if $q \neq T)$. Note that the
events $(E_q : q = 0, \dots, T)$ are disjoint. Let $S_q$ denote the event that
the total drift in the last $q$ pulls of $x^n$ is greater than $\frac12
\mu_{x^n}^d q$, i.e., $\sum_{t \geq T-q+1} d_t > \frac12 \mu_{x^n}^d q$. We can
upper bound the expected leftover budget by conditioning on these events as
follows.

\begin{align}
  \E [B_T] &= \sum_{q=0}^T \Pr [E_q] \  \E [B_T \vert E_q] \\
  &\leq \sum_{q=0}^T \E [B_T \vert E_q] \\
  &= \sum_{q=0}^T \underbrace{\E [B_T \vert E_q, S_q]}_{(a)} \  \underbrace{\Pr [S_q \vert E_q]}_{(b)} + \underbrace{\E [B_T \vert E_q, S_q^c]}_{(c)} \  \underbrace{\Pr [S_q^c \vert E_q]}_{(d)}.
\end{align}

If $q = 0$, then the expected leftover budget is trivially at most a constant.
We can bound the four terms for $q \geq 1$ as follows:
\begin{enumerate}[label=(\alph*)]
  \item We have
  \begin{equation}
    \E [B_T \vert E_q, S_q] \leq c \log q + q
  \end{equation}
  because (i) \texttt{ControlBudget}~(\cref{alg:controlbudget_one_resource})
  pulls $x^0$ or $x^p$ in round $T-q$ if $B_{T-q-1} < \tau_{T-q} = c \log q$;
  and (ii) conditioned on the event $S_q$, the total drift in the last $q$
  rounds can be at most $q$ as the drifts are bounded by $1$.
  \item We have
  \begin{equation}
    \Pr [S_q \vert E_q] \leq \exp \left(- \frac{1}{16} (\mu_{x^n}^d)^2 q \right)
  \end{equation}
  because (i) the sequence of drifts observed from $q$ pulls of the negative
  drift arm $x^n$ is a supermartingale difference sequence; and (ii) by the
  Azuma-Hoeffding inequality, the probability the sum $S_q$ is greater than half
  its expected value is at most $\exp \left( -\frac{1}{16} (\mu_{x^n}^d)^2 q
  \right)$.
  \item We have
  \begin{equation}
    \E [B_T \vert E_q, S_q^c] \leq \left(c \log q + \frac12 \mu_{x^n}^d q \right)
  \end{equation}
  because (i) \texttt{ControlBudget}~(\cref{alg:controlbudget_one_resource})
  pulls $x^0$ or $x^p$ in round $T-q$ if $B_{T-q-1} < \tau_{T-q} = c \log q$;
  and (ii) conditioned on the event $S_q^c$, the total drift in the last $q$
  rounds can be at most $\frac12 \mu_{x^n}^d q$.
  \item We have
  \begin{equation}
    \Pr [S_q^c \vert E_q] \leq 1
  \end{equation}
  trivially.
\end{enumerate}

Therefore,
\begin{equation}
  \E [B_T] \leq \sum_{q=0}^T \underbrace{(c \log q + q) \exp \left(- \frac{1}{16} (\mu_{x^n}^d)^2 q \right)}_{(e)} + \underbrace{\left(c \log q + \frac12 \mu_{x^n}^d q \right)}_{(f)}.
\end{equation}
This summation is a constant in terms of $T$:
\begin{enumerate}
  \item Term (e) is a constant because $c \log q < q$ for $q$ large enough and
  $\sum_{q=1}^\infty q \exp(-aq)$ converges to $\exp(a) (1-\exp(a))^{-2}$.
  \item Term (f) is a constant because this term is negative for $q$ large
  enough as $\mu_{x^n}^d < 0$ and is maximized at $q =
  \frac{2c}{\left|\mu_{x^n}^d\right|}$.
\end{enumerate}

Finally, we can bound the expected leftover budget as
\begin{equation}
  \E [B_T] \leq \tilde{C} = \tilde{O} \left( \left( 1 - \exp \left( \frac{\deltadrift^2}{16} \right) \right)^{-2} + \frac{1}{\deltadrift^2} \right),
\end{equation}
where the last equality follows when $c \geq \frac{6}{\deltadrift^2}$. This
completes the proof.

\section{Proofs for~\cref{subsec:policy_multiple_resources}}
\label{sec:appendix_known_multiple_resources}

%
%
%


\subsection{Proof of~\cref{lemma:gammastar}}

It suffices to show that $\gamma = \frac{\sigma_{\min} \min\{ \deltasupport,
\deltaslack \}}{4 m}$ is a feasible solution
the~\cref{eq:controlbudget_gamma_t}.

First, we show that $p = D^{-1} (b + \gamma s_t) \geq 0$. For each $x \in X$,
\begin{align*}
  e_x^T D^{-1} (b + \gamma s_t) &= e_x^T D^{-1} b + \gamma e_x^T D^{-1} s_t \\
  &= p^*_x + \gamma e_x^T D^{-1} s_t \\
  &\geq \deltasupport - \gamma \| D^{-1} s_t \|_2 \\
  &\geq \deltasupport - \gamma \frac{1}{\sigma_{\min}} \sqrt{m} \\
  &\geq 0.
\end{align*}

Second, we show that for any non-binding resource $j$, $d_j^T D^{-1} (b + \gamma
s_t) \geq \frac{\deltaslack}{2}$:
\begin{align*}
  d_j^T D^{-1} (b + \gamma s_t) &= \sum_{x \in X} d_j(x, \mu) p^*_x + \gamma d_j^T D^{-1} s_t \\
  &\geq \deltaslack - \gamma |d_j^T D^{-1} s_t| \\
  &\geq \deltaslack - \gamma \|d_j^T\|_2 \| D^{-1} \|_2 \|s_t\|_2 \\
  &\geq \deltaslack - \gamma \frac{1}{\sigma_{\min}} m \\
  &\geq \frac{\deltaslack}{2} \\
  &\geq \frac{\gamma}{2},
\end{align*}
where the last inequality follows because $\sigma_{\min}, \deltaslack,
\deltasupport < 1$.


\subsection{Proof of~\cref{lemma:null_arm_pulls_multiple_resources}}

Divide the $T$ rounds into two phases: $P_1 = \{ 1, \dots, T -
\exp(\nicefrac{3}{c})$ and $P_2 = \{ 1, \dots, T \} \setminus P_1$. Note that
$P_2$ consists of $\exp(\nicefrac{3}{c}) = O(\exp(\gamma^*)) = O(1)$ rounds,
where the last equality follows because $\gamma^*$ is bounded by $1$. Therefore,
the expected number of null arm pulls in this phase is $O(1)$ and it suffices to
bound the expected number of null arm pulls in $P_1$.


We can write the expected number of rounds in which there exists a resource
whose budget is less than $1$ as
\begin{equation}
  \E \left[ \sum_{t=1}^T \sum_{j \in \calJ} \indicator[B_{t-1,j} < 1] \right] = \sum_{j \in \calJ} \underbrace{\E \left[ \sum_{t=1}^T \indicator[B_{t-1,j} < 1] \right]}_{(a)}.
\end{equation}
We can bound term (a) above the same way as in the proof
of~\cref{lemma:null_arm_pulls_one_resource}~(\cref{subsec:appendix_lemma_null_arm_pulls_one_resource})
with $\deltadrift$ replaced by $\gamma^*$.. Therefore,
\begin{equation}
  \E \left[ \sum_{t=1}^T \sum_{j \in \calJ} \indicator[B_{t-1,j} < 1] \right] \leq m \tilde{C} \left( 1 + \frac{\pi^2}{6} \right),
\end{equation}
where $\tilde{C}$ is defined in~\cref{eq:one_arm_pos_drift_tildec}.


\subsection{Proof of~\cref{lemma:leftover_budget_multiple_resources}}

Consider an arbitrary resource $j \in J^*$. Recall the vector $s_t$ defined in
\texttt{ControlBudget}~(\cref{alg:controlbudget}). If $i$ denote the row
corresponding to resource $j$, then the $i$th entry of $s_t$, denoted by
$s_t(i)$, is $-1$ if $B_{t-1, j} < \tau_t$ and $+1$ otherwise.

Let $E_q$ denote the event that the $s_t(i)$ is equal to $-1$ consecutively in
exactly the last $q$ rounds, i.e., $s_t(i) = -1$ for all $t \geq T-q+1$ and
$s_t(i) = +1$ for $t = T-q$ (if $q \neq T$). Note that the events $(E_q : q = 0,
\dots, T)$ are disjoint. Let $S_q$ denote the event that the total drift for $j$
in the last $q$ rounds is greater than $\frac12 (-\gamma^*) q$, i.e., $\sum_{t
\geq T-q+1} d_t > \frac12 (-\gamma^*) q$. We can upper bound the expected
leftover budget of resource $j$ by conditioning on these events as follows.

\begin{align}
  \E [B_{T,j}] &= \sum_{q=0}^T \Pr [E_q] \  \E [B_{T,j} \vert E_q] \\
  &\leq \sum_{q=0}^T \E [B_{T,j} \vert E_q] \\
  &= \sum_{q=0}^T \underbrace{\E [B_{T,j} \vert E_q, S_q]}_{(a)} \  \underbrace{\Pr [S_q \vert E_q]}_{(b)} + \underbrace{\E [B_{T,j} \vert E_q, S_q^c]}_{(c)} \  \underbrace{\Pr [S_q^c \vert E_q]}_{(d)}.
\end{align}

If $q = 0$, then the expected leftover budget is trivially at most a constant.
We can bound the four terms for $q \geq 1$ as follows:
\begin{enumerate}[label=(\alph*)]
  \item We have
  \begin{equation}
    \E [B_{T,j} \vert E_q, S_q] \leq c \log q + q
  \end{equation}
  because (i) \texttt{ControlBudget}~(\cref{alg:controlbudget}) sets $s_t(i) =
  +1$ in round $T-q$ if $B_{T-q-1,j} < \tau_{T-q} = c \log q$; and (ii)
  conditioned on the event $S_q$, the total drift in the last $q$ rounds can be
  at most $q$ as the drifts are bounded by $1$.
  \item We have
  \begin{equation}
    \Pr [S_q \vert E_q] \leq \exp \left(- \frac{1}{16} (\gamma^*)^2 q \right)
  \end{equation}
  because (i) the sequence of drifts observed in rounds $t \geq T-q+1$ is a
  supermartingale difference sequence with $\E [d_{s,j} \vert d_{T-q+1,j},
  \dots, d_{s-1,j}] \leq -\gamma^*$; and (ii) by the Azuma-Hoeffding inequality,
  the probability the sum $S_q$ is greater than half its expected value is at
  most $\exp \left( -\frac{1}{16} (\gamma^*)^2 q \right)$.
  \item We have
  \begin{equation}
    \E [B_{T,j} \vert E_q, S_q^c] \leq \left(c \log q + \frac12 (-\gamma^*) q \right)
  \end{equation}
  because (i) \texttt{ControlBudget}~(\cref{alg:controlbudget}) sets $s_t(i) =
  +1$ in round $T-q$ if $B_{T-q-1,j} < \tau_{T-q} = c \log q$; and (ii)
  conditioned on the event $S_q^c$, the total drift in the last $q$ rounds can
  be at most $\frac12 (-\gamma^*) q$.
  \item We have
  \begin{equation}
    \Pr [S_q^c \vert E_q] \leq 1
  \end{equation}
  trivially.
\end{enumerate}

Therefore,
\begin{equation}
  \E [B_{T,j}] \leq \sum_{q=0}^T \underbrace{(c \log q + q) \exp \left(- \frac{1}{16} (\gamma^*)^2 q \right)}_{(e)} + \underbrace{\left(c \log q + \frac12 (-\gamma^*) q \right)}_{(f)}.
\end{equation}
This summation is a constant in terms of $T$:
\begin{enumerate}
  \item Term (e) is a constant because $c \log q < q$ for $q$ large enough and
  $\sum_{q=1}^\infty q \exp(-aq)$ converges to $\exp(a) (1-\exp(a))^{-2}$.
  \item Term (f) is a constant because this term is negative for $q$ large
  enough and is maximized at $q = \frac{2c}{\gamma^*}$.
\end{enumerate}

Finally, we can bound the expected leftover budget as
\begin{equation}
  \E [B_{T,j}] \leq \tilde{C} = \tilde{O} \left( \left( 1 - \exp \left( \frac{\gamma^{*2}}{16} \right) \right)^{-2} + \frac{1}{\gamma^{*2}} \right),
\end{equation}
where the last equality follows when $c \geq \frac{6}{\gamma^{*2}}$. This
completes the proof.

\section{Proofs for~\cref{sec:unknown_distributions}}
\label{sec:appendix_unknown}


\subsection{Proof of~\cref{lemma:learning_clean_event}}

It suffices to show that the complement of the clean event occurs with
probability at most $5mT^{-2}$.

For (i) in the definition of the clean event~(\cref{def:clean_event}), by taking
a union bound over the components of the outcome vector and using Azuma-Hoeffding
inequality, we have
\begin{align}
  \mu_x^o \notin [\lcb_t(x), \ucb_t(x)] &\leq 2 (m+1) \exp \left( - 2 n_t(x) \rad_t(x)^2 \right) \\
  &\leq 4m \exp \left( - 2 n_t(x) \frac{8 \log T}{n_t(x)} \right) \\
  &\leq 4m T^{-2}.
\end{align}

For (ii) in the definition of the clean event~(\cref{def:clean_event}), a
similar approach works. Let $S_{n,j}$ denote the sum of the drifts for resource
$j \in \calJ$ after $n$ pulls of the null arm $x^0$. By the union bound and
Azuma-Hoeffding inequality,
\begin{align}
  \Pr \left[ \exists j \in \calJ \text{ s.t. } S_{n,j} < w \right] &\leq m \exp \left( - \frac14 w \mu_{x^0}^d \right) \\
  &\leq m \exp \left( - \frac14 w \deltadrift \right) \\
  &\leq m \exp \left( - \frac14 \frac{1024 k m^2 \log T}{\deltadrift^2 \sigma_{\min}^2} \deltadrift \right) \\
  &= m \exp \left( - \frac{256 k m^2 \log T}{\deltadrift \sigma_{\min}^2} \right) \\
  &\leq m \exp \left( - 256 k m^2 \log T \right) \label{eq:proof_learning_clean_event_1} \\
  &\leq m \exp \left( - 256 \log T \right) \\
  &\leq m T^{-2},
\end{align}
where~\cref{eq:proof_learning_clean_event_1} follows because $\deltadrift \in
(0, 1]$ and $\sigma_{\min} \in (0, 1)$.  This shows that the probability of the
complement of the clean event is at most $5 m T^{-2}$ and completes the proof.



\subsection{Proof of~\cref{lemma:learning_lp_confidence_interval}}

We will prove the lemma for $\opt_\lp$ because the other cases are similar.
Simplifying and overloading notation for this proof, we denote the probability
simplex over $k$ dimensions as $\Delta_k$, and the vector of expected rewards,
the matrix of expected drifts and the right-hand side of the budget constraints
as
\begin{equation}
  r = \begin{bmatrix} \mu^r_1 \\ \vdots \\ \mu^r_k \end{bmatrix},
  \ 
  D = \begin{bmatrix} \mu^{d,1}_1 & \dots & \mu^{d,1}_k \\ & \ddots & \\ \mu^{d,m}_1 & \dots & \mu^{d,m}_k \end{bmatrix},
  \ 
  b = -\frac{B}{T}\textbf{1}.
\end{equation}
We will use $\bar{r}$ and $\bar{D}$ to denote the empirical versions of the rewards and drifts. We can write
\begin{align*}
  \opt_\lp &= \max_{p \in \Delta_k} r^T p & \text{ s.t. } D p \geq b, \\
  \ucb_t(\opt_\lp) &= \max_{q \in \Delta_k} (\bar{r} + \rad_t)^T q & \text{ s.t. } (\bar{D} + \rad_t) q \geq b \\
  &\leq \max_{q \in \Delta_k} (r + 2 \rad_t)^T q & \text{ s.t. } (D + 2 \rad_t) q \geq b \\
  &\leq 2 \rad_t + \max_{q \in \Delta_k} r^T q & \text{ s.t. } D q \geq b -  2 \rad_t,
\end{align*}
where the second-last inequality follows because we are conditioning on the
clean event. Therefore, using $D'$ and $b'$ to denote the submatrix and
subvector corresponding to the binding constraints, we have
\begin{align*}
  \ucb_t(\opt_\lp) - \opt_\lp &\leq 2 \rad_t + |r^T p - r^T q| \\
  &\leq 2 \rad_t + |r^T (D')^{-1} b' - r^T (D')^{-1} (b' - 2 \rad_t)| \\
  &\leq 2 \rad_t + \|r\|_2 \|(D')^{-1}\|_2 \|2 \rad_t\|_2 \\
  &\leq 2 \rad_t + 2 m \rad_t \frac{1}{\sigma_{\min}} \\
  &\leq \frac{4m}{\sigma_{\min}} \rad_t,
\end{align*}
where last inequality follows because $\sigma_{\min} < 1 \leq m$. Since the LCB is defined by subtracting $\rad_t$ from the empirical means, we obtain the same upper bound on $\opt_\lp - \lcb_t(\opt_\lp)$ and using the triangle inequailty completes the proof.


\subsection{Proof of~\cref{theorem:regret_explorethencontrolbudget}}
\label{subsec:appendix_theorem_regret_explorethencontrolbudget}

Since the complement of the clean event occurs with probability at most $O(m
T^{-2})$ and contributes $O(T)$ to the regret, it suffices to bound the regret
conditioned on the clean event. So, condition on the clean event for the rest of
the proof. Phase one contributes at most
\begin{equation}
  O \left( \frac{k m^2}{\min \{ \deltadrift^2, \sigma_{\min}^2 \} \Delta^2} \right) \cdot \log T
\end{equation}
to the regret by~\cref{corollary:learning_phase_one_rounds}. Phase two
contributes at most
\begin{equation}
  O \left( \frac{k}{\gamma^{*2}} \right) \log T
\end{equation}
to the regret.

Observe that after phase two, $\rad_t(x) \leq \frac{\gamma^{*2}}{2}$ for all $x
\in X^*$. Combining this
with~\cref{eq:controlbudget_gamma_t,eq:etcb_gamma_t_1,eq:etcb_gamma_t_2,eq:etcb_gamma_t_3},
we have that $(\gamma^*, D^{-1}(b + \gamma^* s_t))$ is a feasible solution to
the optimization problem solved by
\texttt{ExploreThenControlBudget}~(\cref{alg:explorethencontrolbudget}).
Therefore, $(\gamma_t, p_t)$ ensure that there is drift of magnitude at least
$\frac{\gamma^*}{8}$ in the ``correct directions''. As noted in the end
of~\cref{subsec:policy_multiple_resources}, the regret analysis of
\texttt{ControlBudget}~(\cref{alg:controlbudget}) requires the algorithm to know
$X^*$, $J^*$, and find a probability vector $p_t$ that ensures drifts bounded
away from zero  in the ``correct directions''. Therefore,
by~\cref{theorem:regret_controlbudget}, phase three contributes at most
$\tilde{C'}$ to the regret, where $\tilde{C}$ is the constant
in~\cref{theorem:regret_controlbudget}.  Combining the contribution from the
three phases, we have that
\begin{equation}
  R_T(\texttt{ExploreThenControlBudget}) \leq \tilde{C} \cdot \log T,
\end{equation}
where $\gamma^*$ (defined in~\cref{lemma:gammastar}) and $\tilde{C}$ are
constants with
\begin{equation}
  \tilde{C} = O \left( \frac{k m^2}{\min \{ \deltadrift^2, \sigma_{\min}^2 \} \Delta^2} + k (\gamma^*)^{-2} + \tilde{C}' \right).
\end{equation}


\subsection{Proof of~\cref{theorem:reduction}}
\label{subsec:appendix_reduction}

Note that $\bwk$ is not automatically a special case of our model because of our
assumption that the null arm has strictly positive drift for every resource. In
this section we present a reduction from $\bwk$ with $\frac{B}{T}$ bounded away
from $0$ to our model. We show that our results imply a logarithmic regret bound
for $\bwk$ under certain assumptions.

\paragraph{Reduction}
Assume we are given an instance of $\bwk$ with $\frac{B}{T} \geq \deltadrift >
0$. (Existing results on logarithmic regret for $\bwk$ also assume the ratio of
the initial budget to the time horizon is bounded away from
$0$~\citep{li2021symmetry}.) We will reduce the given $\bwk$ instance to a
problem in our model. The reduction initializes an instance of
\texttt{ExploreThenControlBudget}~(\cref{alg:explorethencontrolbudget}) running
in a simulated environment with the same set of arms as in the given $\bwk$
instance, plus an additional null arm whose drift is equal to $\deltadrift$
deterministically for each resource. The reduction will maintain two time
counters: $t_a$ is the actual number of time steps that have elapsed in the
$\bwk$ prlblem, and $t_s$ is the number of time steps that have elapsed in the
simulation environment in which~\cref{alg:explorethencontrolbudget} is running.
Likewise, there are two vectors that track the remaining budget: $B_a$ is the
remaining budget in the actual $\bwk$ problem our reduction is solving, while
$B_s$ is the remaining budget in the simulation environment. These two budget
vectors will always be related by the equation
\begin{equation}\label{eq:appendix_reduction_update}
  B_s = B_a - T \deltadrift \textbf{1} + t_s \deltadrift \textbf{1}.
\end{equation}
In particular, the initial budget of each resource is initialized (at simulated
time $t_s = 0$) to $B - T \deltadrift$.

Each step of the reduction works as follows. We
call~\cref{alg:explorethencontrolbudget} to simulate one time step in the
simulated environment. If~\cref{alg:explorethencontrolbudget} recommends to pull
a non-null arm $x$, we pull arm $x$, increment both of the time counters ($t_a$
and $t_s$), and update the vector of remaining resource amounts, $B_a$,
according to the resources consumed by arm $x$.
If~\cref{alg:explorethencontrolbudget} recommends to pull the null arm, we do
not pull any arm, and we leave $t_a$ and $B_a$ unchanged; however, we still
increment the simulated time counter $t_s$. Finally, regardless of whether a
null or non-null arm was pulled, we update $B_s$ to
satisfy~\cref{eq:appendix_reduction_update}.

\paragraph{Correctness}
Since the reduction pulls the same sequence of non-null arms
as~\cref{alg:explorethencontrolbudget} until the $\bwk$ stopping condition is
met and the additional pulls of the null arm in the simulation environment yield
zero reward, the total reward in the actual $\bwk$ problem equals the total
reward earned in the simulation environment at the time when the $\bwk$ stopping
condition is met and the reduction ceases running.
Since~\cref{alg:explorethencontrolbudget} maintains the invariant that $B_s$ is
a nonnegative vector,~\cref{eq:appendix_reduction_update} ensures that $B_a$
will also remain nonnegative as long as $t_s \geq T$ must
hold.~\cref{theorem:regret_explorethencontrolbudget} ensures that the total
expected reward earned in the simulation environment and hence, also in the
$\bwk$ problem itself, is bounded below by $T \cdot \opt_\lp - \tilde{C} \cdot
\log T$, where $\tilde{C}$ is the constant
in~\cref{theorem:regret_explorethencontrolbudget} and $\opt_\lp$ denotes the
optimal value of the LP relaxation~(\cref{eq:lp_relaxation}) for the simulation
environment.

We would like to show that this implies the regret of the reduction (with
respect to the LP relaxation of $\bwk$) is bounded by $\tilde{C} \cdot \log T$.
To do so, we must show that the LP relaxations of the original $\bwk$ problem
and the simulation environment have the same optimal value. Let $\mu_x^r$ and
$\mu_x^{d,j}$ denote the expected reward and expected drifts in the actual
$\bwk$ problem with arm set $\calX$, and let $\hat{\mu}_x^r$ and
$\hat{\mu}_x^{d,j}$ denote the expected reward and drifts in the simulation
environment with arm set $\calX^+ = \calX \cup \{ x^0 \}$. The two LP
formulations are as follows.
\begin{equation*}
{\small 
\begin{array}{l@{\quad} r l r r}
  \max\limits_{p} & \sum\limits_{x \in \calX} p_x \mu_x^r & & \\
  \text{s.t.}     & \sum\limits_{x \in \calX} p_x \mu_x^{d,j} &\geq -\frac{B}{T} & \forall j \in \calJ, \\
                  & \sum\limits_{x \in \calX} p_x & \leq 1 & \\
                  & p_x &\geq 0 & \forall x \in \calX.
\end{array}
\quad
\begin{array}{l@{\quad} r l r r}
  \max\limits_{p} & \sum\limits_{x \in \calX} p_x \hat{\mu}_x^r & & \\
  \text{s.t.}     & \sum\limits_{x \in \calX} p_x \hat{\mu}_x^{d,j} &\geq -\frac{B}{T} - \deltadrift & \forall j \in \calJ, \\
                  & \sum\limits_{x \in \calX^+} p_x &= 1 & \\
                  & p_x &\geq 0 & \forall x \in \calX^+.
\end{array}
}
\end{equation*}
The differences between the two LP formulations lie in substituting $\hat{\mu}$
for $\mu$, substituting $\calX^+$ for $\calX$, and transforming the inequality
constraint $\sum_{x \in \calX} p_x \leq 1$ into an equality constraint $\sum_{x
\in \calX^+} p_x = 1$. We know that $\mu_x^r = \hat{\mu}_x^r$ for every $x \in
\calX$ and $\hat{\mu}_{x^0}^r = 0$.  Furthermore,  $\hat{\mu}_x^{d,j}$ denotes
the expected drift of resource $j$ in the simulation environment when arm $x$ is
pulled.  This can be written as the sum of two terms: drift $\mu_x^{d,j}$ is the
expectation of the (non-positive) quantity added to the $j$th component of
budget vector $B_a$ when pulling arm $x$ in the actual $\bwk$ environment; in
addition to this non-positive drift, there is a deterministic positive drift of
$\deltadrift$ due to incrementing the simulation time counter $t_s$ and
recomputing $B_s$ using~\cref{eq:appendix_reduction_update}. Hence,
$\hat{\mu}_x^{d,j} = \mu_x^{d,j} + \deltadrift$ for all $x \in \calX$ and $j \in
\calJ$. Furthermore, $\hat{\mu}_{x^0}^{d,j} = \deltadrift$.  Hence, for any
vector $\vec{p}$ representing a probability distribution on $\calX^+$, we have
\begin{equation}\label{eq:appendix_reduction_comparing_lp_constraints}
  \sum\limits_{x \in \calX^+} p_x \hat{\mu}_x^{d,j} = \left( \sum\limits_{x \in \calX} p_x \hat{\mu}_x^{d,j} \right) + \deltadrift.
\end{equation}
Accordingly, a vector $\vec{p}$ satisfies the constraints of the $\bwk$ LP
relaxation above if and only if the probability vector on $\calX^+$ obtained
from $\vec{p}$ by setting $p_{x^0} = 1 - \sum_{x \in \calX} p_x$ satisfies the
constraints of the second LP relaxation above. This defines a one-to-one 
correspondence between the sets of vectors feasible for the two LP formulations.
Furthermore, this one-to-one correspondence preserves the value of the objective
function because $\hat{\mu}_x^r = \mu_x^r$ for $x \in \calX$ and
$\hat{\mu}_{x^0}^r=0$. Thus, the optimal value of the two linear programs is
the same. This completes the proof.

\section{Experiments}
\label{sec:appendix_experiments}


In this section we present some simple experimental results.
\footnote{The code and data are available at \url{https://github.com/raunakkmr/non-monotonic-resource-utilization-in-the-bandits-with-knapsacks-problem-code}.}
For simplicity, we only consider Bernoulli distributions, i.e., rewards are
supported on $\{ 0, 1 \}$, a positive drift arm's drifts are supported on $\{ 0,
1 \}$, and a negative drift arm's drifts are supported on $\{ 0, -1 \}$. We
generate the data for the experiments as follows:
\begin{itemize}
  \item \cref{fig:experimental_results_cb} plot (a): We set $T = 25k, B = 0, n = 2$ and $m = 1$. The expected reward and drifts for the arms are: $(0;0.1), (0.8;0.4)$. The LP solution is supported on a single positive drift arm.
  \item \cref{fig:experimental_results_cb} plot (b): We set $T = 25k, B = 400, n = 2$ and $m = 1$. The expected reward and drifts for the arms are: $(0;0.4), (0.8;-0.3)$. The LP solution is supported on the null arm and the negative drift arm.
  \item \cref{fig:experimental_results_cb} plot (c): We set $T = 25k, B = 400, n = 3$ and $m = 1$. The expected reward and drifts for the arms are: $(0;0.4), (0.8;-0.3), (0.1;0.3)$. The LP solution is supported on the positive drift arm and the negative drift arm.
  \item \cref{fig:experimental_results_cb} plot (d): We set $T = 25k, B = 3, n = 3$ and $m = 2$. The expected reward and drifts for the arms are: $(0;0.1,0.08), (0.8;-0.2,-0.25), (0.1;0.4,0.5)$. The LP solution is supported on the positive drift arm and the negative drift arm.
  \item \cref{fig:experimental_results_etcb} plot (a) and (b): We set $T = 150k, B = 10, n = 3$ and $m = 1$. The expected reward and drifts for the arms are: $(0;0.9), (0.8;-0.6), (0.1;0.7)$. The LP solution is supported on the positive drift arm and the negative drift arm.
\end{itemize}

\begin{figure}
  \centering
  \begin{subfigure}[b]{0.4\textwidth}
      \centering
      \includegraphics[width=\textwidth]{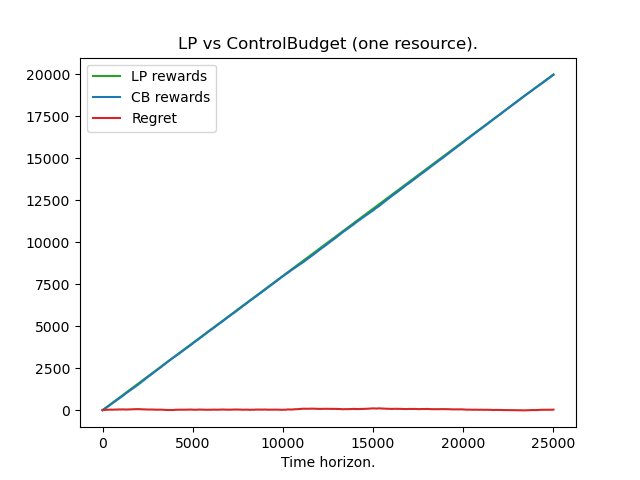}
      \caption{\texttt{ControlBudget} with one resource - case 1 (positive drift arm)}
      \label{fig:mdp_case_1_easy}
  \end{subfigure}
  \hfill
  \begin{subfigure}[b]{0.4\textwidth}
      \centering
      \includegraphics[width=\textwidth]{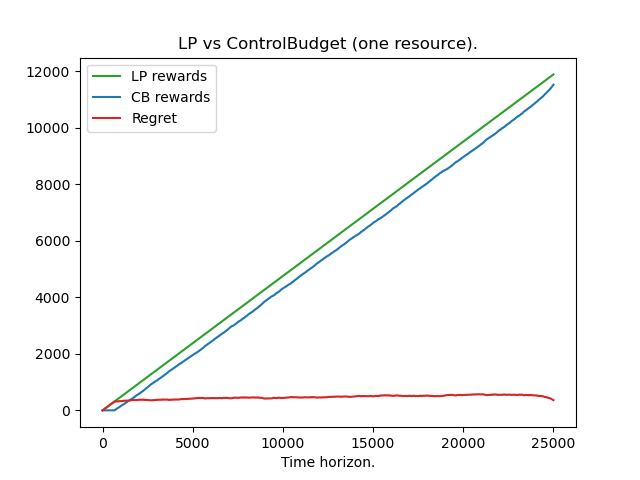}
      \caption{\texttt{ControlBudget} with one resource - case 2 (null plus negative drift arm)}
      \label{fig:mdp_case_2_easy}
  \end{subfigure}
  \hfill
  \begin{subfigure}[b]{0.4\textwidth}
      \centering
      \includegraphics[width=\textwidth]{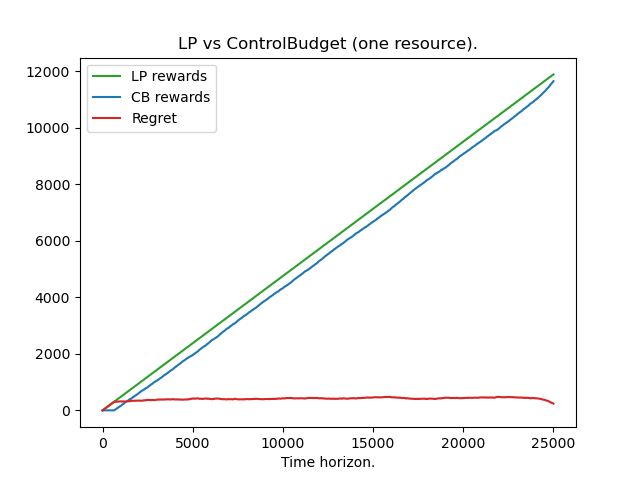}
      \caption{\texttt{ControlBudget} with one resource - case 3 (positive plus negative drift arm)}
      \label{fig:mdp_case_3_easy}
  \end{subfigure}
  \hfill
  \begin{subfigure}[b]{0.4\textwidth}
      \centering
      \includegraphics[width=\textwidth]{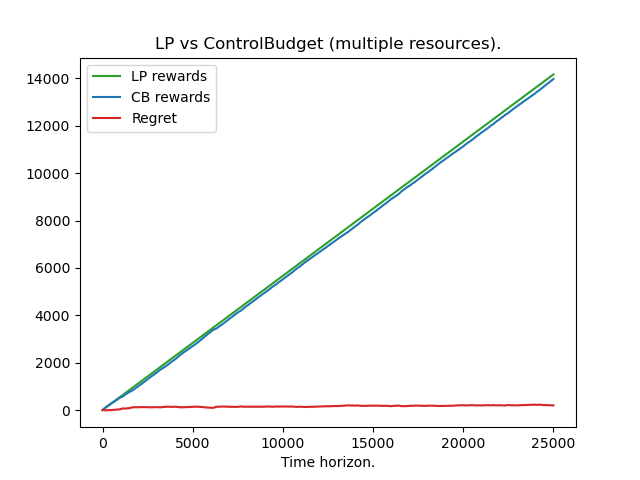}
      \caption{\texttt{ControlBudget} with multiple resources}
      \label{fig:mdp_multiple_resources_easy}
  \end{subfigure}
     \caption{Regret of \texttt{ControlBudget} on a variety of test cases.}
     \label{fig:experimental_results_cb}
\end{figure}

\begin{figure}
  \centering
  \begin{subfigure}[b]{0.4\textwidth}
      \centering
      \includegraphics[width=\textwidth]{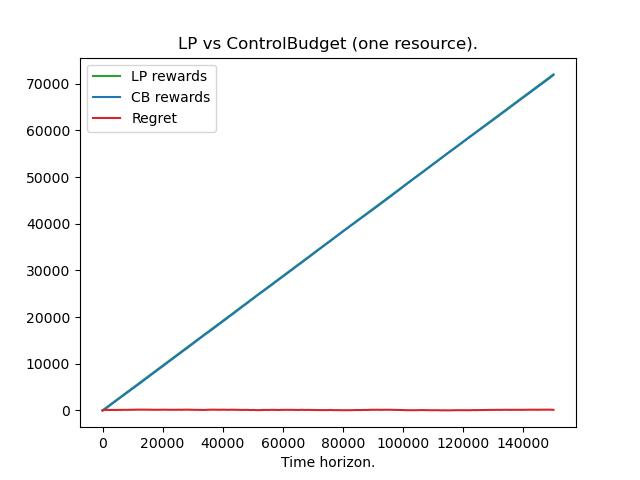}
      \caption{\texttt{ControlBudget}}
      \label{fig:learning_cb}
  \end{subfigure}
  \begin{subfigure}[b]{0.4\textwidth}
      \centering
      \includegraphics[width=\textwidth]{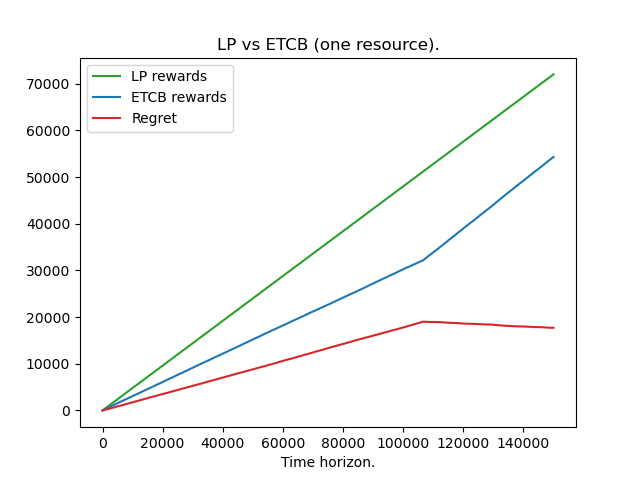}
      \caption{\texttt{ExploreThenControlBudget}}
      \label{fig:learning_etcb}
  \end{subfigure}
     \caption{Regret of \texttt{ControlBudget} and \texttt{ExploreThenControlBudget} on the same test case. We modify \texttt{ExploreThenControlBudget} to use the empirical means instead of UCB/LCB estimates for phase one as described in ~\cref{sec:appendix_experiments}.}
     \label{fig:experimental_results_etcb}
\end{figure}

As our plots show~(\cref{fig:experimental_results_cb}), our MDP policy,
\texttt{ControlBudget}, performs quite well and achieves constant regret.

Our learning algorithm does not perform as well empirically due to large
constant factors. Specifically, the number of rounds required for the confidence
radius to be small enough for phase one to successfully identify $X^*$ and $J^*$
is too large. In our simple test cases, if we simply consider the empirical
means, which are very close to the true means, instead of the UCB/LCB estimates
for phase one, then the learning algorithm performs as expected: it achieves
logarithmic regret by spending a logarithmic number of rounds identifying $X^*$
and $J^*$, and achieves constant regret thereafter~(\cref{fig:experimental_results_etcb}).



\end{document}